%% file: neurips_2026.tex
\documentclass{article}

\PassOptionsToPackage{numbers, compress}{natbib}
\usepackage[preprint]{neurips_2026}

\usepackage[utf8]{inputenc} % allow utf-8 input
\usepackage[T1]{fontenc}    % use 8-bit T1 fonts
\usepackage{hyperref}       % hyperlinks
\hypersetup{
  colorlinks=true,
  citecolor={blue!60!black},
  linkcolor={red!60!black},
  urlcolor={blue!60!black},
}
\usepackage{url}            % simple URL typesetting
\usepackage{booktabs}       % professional-quality tables
\usepackage{multirow}       % multi-row cells in tables
\usepackage{caption}        % \captionof outside floats
\usepackage{subcaption}     % \subcaption inside minipage
\usepackage{amsfonts}       % blackboard math symbols
\usepackage{amssymb}        % additional math symbols (triangles etc.)
\usepackage{amsmath}
\usepackage{amsthm}
\usepackage{nicefrac}       % compact symbols for 1/2, etc.
\usepackage{microtype}      % microtypography
\usepackage{xcolor}         % colors
\usepackage{colortbl}       % \rowcolor in tables
\usepackage{graphicx}       % figures
\usepackage{wrapfig}        % wrap figures
\usepackage{pifont}          % ding symbols (checkmark / cross)
\usepackage{array}           % m{} column type
\usepackage{placeins}        % keep floats near sections
\usepackage{enumitem}        % customizable list environments
\usepackage{algorithm}       % algorithm float environment
\usepackage{algorithmic}     % algorithmic pseudocode
\usepackage[most]{tcolorbox} % colored boxes for findings/roadmap
\usepackage{tikz}            % for custom icons

\newcommand{\up}[1]{\textcolor{green!50!black}{\small$\blacktriangle$\,#1}}
\newcommand{\down}[1]{\textcolor{red!70!black}{\small$\blacktriangledown$\,#1}}

\newtcolorbox{findingbox}[1][]{%
  enhanced, colback=blue!2, colframe=dareblue!60,
  boxrule=0.6pt, arc=3pt, left=6pt, right=6pt, top=5pt, bottom=5pt,
  fonttitle=\bfseries\small, coltitle=dareblue,
  attach boxed title to top left={yshift=-2mm, xshift=4mm},
  boxed title style={colback=white, colframe=white, boxrule=0pt},
  title={#1}
}
\newtcolorbox{roadmapbox}{%
  enhanced, colback=gray!4, colframe=gray!50,
  boxrule=0.4pt, arc=2.5pt, left=6pt, right=6pt, top=4pt, bottom=4pt
}

\newcommand{\iconeasy}{\raisebox{-0.15em}{\includegraphics[height=1em]{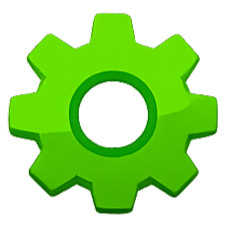}}\xspace}
\newcommand{\iconmedium}{\raisebox{-0.15em}{\includegraphics[height=1em]{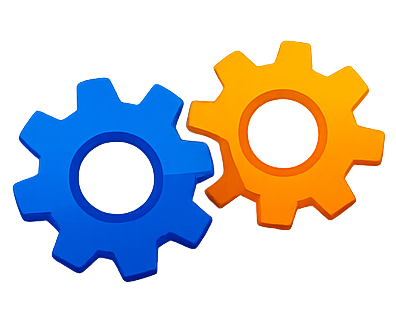}}\xspace}
\newcommand{\iconhard}{\raisebox{-0.15em}{\includegraphics[height=1em]{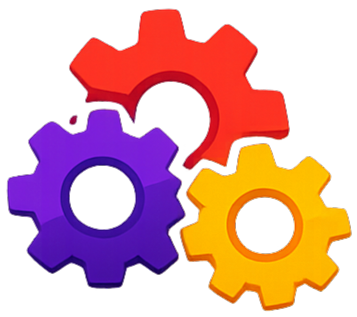}}\xspace}
\newcommand{\iconeasybig}{\raisebox{-0.2em}{\includegraphics[height=1.2em]{figure/icon/easy.png}}\xspace}
\newcommand{\iconmediumbig}{\raisebox{-0.2em}{\includegraphics[height=1.2em]{figure/icon/medium.png}}\xspace}
\newcommand{\iconhardbig}{\raisebox{-0.2em}{\includegraphics[height=1.2em]{figure/icon/hard.png}}\xspace}

\newtheorem{theorem}{Theorem}[section]
\newtheorem{proposition}[theorem]{Proposition}

\newcommand{\appref}[1]{\hyperref[#1]{Appendix~\ref*{#1}}}
\newcommand{\cmark}{\textcolor{green!50!black}{\ding{51}}}
\newcommand{\xmark}{\textcolor{red}{\ding{55}}}

\usepackage{xspace}
\usepackage{xcolor}
\definecolor{dareblue}{HTML}{1F4E79}
\newcommand{\ours}{\textcolor{dareblue}{\textsc{Dare}}\xspace}
\newcommand{\circmarklight}[2]{%
  \tikz[baseline=(char.base)]{
    \node[shape=circle, draw=#1, inner sep=1pt, font=\scriptsize\bfseries, text=#1] (char) {#2};
  }%
}

\title{\textbf{\ours}: Difficulty-Adaptive Reinforcement Learning with Co-Evolved Difficulty Estimation}

\author{
  \textbf{Yang Zhou}$^{1*}$ \quad
  \textbf{Can Jin}$^{1*\dag}$ \quad
  \textbf{Zihan Dong}$^{1}$ \quad
  \textbf{Zhepeng Wang}$^{2}$ \quad
  \textbf{Yanting Yang}$^{1}$ \\
  \textbf{Shiyu Zhao}$^{1}$ \quad
  \textbf{Lei Li}$^{3}$ \quad 
  \textbf{Runxue Bao}$^{4}$ \quad  
  \textbf{Yaochen Xie}$^{2}$ \quad 
  \textbf{Dimitris N. Metaxas}$^{1}$ \\
  \\ 
  $^1$Rutgers University \quad
  $^2$Amazon \quad
  $^3$Washington University \quad
  $^4$Google
}

\makeatletter
\renewcommand{\@makefnmark}{}
\makeatother

\begin{document}

\maketitle
\footnotetext{$^{*}$ Equal contribution. $^{\dagger}$ Lead architect. Correspondence to Yang Zhou \href{mailto:eta.yang@rutgers.edu}{<eta.yang@rutgers.edu>} and Can Jin \href{mailto:can.jin@rutgers.edu}{<can.jin@rutgers.edu>}.}

\input{sec/0_abstract}

\input{sec/1_introduction}

\input{sec/2_related_work}
\input{sec/3_preliminary}

\input{sec/4_method}

\input{sec/5_experiment}

\input{sec/6_conclusion}

\clearpage

\bibliographystyle{plainnat} 
\bibliography{neurips_2026}

%%%%%%%%%%%%%%%%%%%%%%%%%%%%%%%%%%%%%%%%%%%%%%%%%%%%%%%%%%%%
\clearpage
\appendix

\input{sec/appendix}

% \section{Technical Appendices and Supplementary Material}
% Technical appendices with additional results, figures, graphs and proofs may be submitted with the paper submission before the full submission deadline (see above), or as a separate PDF in the ZIP file below before the supplementary material deadline. There is no page limit for the technical appendices.

%%%%%%%%%%%%%%%%%%%%%%%%%%%%%%%%%%%%%%%%%%%%%%%%%%%%%%%%%%%%

\clearpage

\end{document}

%% file: sec/0_abstract.tex
\begin{abstract}

Reinforcement learning improves the reasoning ability of large language models but remains costly and sample-inefficient, as many rollouts provide weak learning signals. Difficulty-aware data selection methods attempt to address this by prioritizing moderately difficult prompts, yet our analysis reveals three limitations: difficulty estimates become inaccurate under policy drift, data selection alone yields limited final-performance gains, and inference efficiency remains largely unchanged. These findings suggest that efficient and effective RL requires more than filtering by difficulty: the policy should learn to solve hard tasks while producing concise responses for easy ones. To this end, we propose \ours, a unified framework that co-evolves difficulty estimation with the policy via self-normalized importance sampling, maintains diverse difficulty coverage through a symmetric Beta sampling distribution, and applies tailored training strategies across difficulty tiers with adaptive compute allocation. Extensive experiments across multiple models and domains demonstrate that \ours consistently outperforms existing methods in training efficiency, final effectiveness, and inference efficiency, producing more concise responses on easy tasks while improving correctness on hard ones. Code is available at \url{https://github.com/EtaYang10th/DARE}.

\end{abstract}

%% file: sec/1_introduction.tex
% \red{add PPO experiments}

% \red{If training solely on hard tasks, the model might forget how to solve easy tasks that previously can answer correctly}

% Embedding-based predictors~\cite{sun2025dots} estimate difficulty based on similarity to a fixed prompt space. Entropy-based methods use sequence-level entropy as a proxy for task difficulty \citep{pang2026edco}. Empirical success rates over multiple rollouts also serve as a measure of prompt difficulty \citep{tong2024dart}, and recent work explores LLMs themselves as difficulty judges \citep{tabib2025toward}. 

\section{Introduction}
\label{sec:intro}
Reinforcement learning (RL) has emerged as a powerful paradigm for improving the reasoning capabilities of large language models (LLMs), with algorithms such as GRPO~\citep{guo2025deepseek,shao2024deepseekmath} driving substantial gains on mathematical, coding, and scientific benchmarks~\citep{shao2024deepseekmath,jain2025livecodebench_iclr,rein2023gpqa}. However, existing RL methods remain highly resource-intensive and sample-inefficient~\citep{yue2025does,tang2026highdataeff,sun2025dots}, largely due to repeated rollout generation, in which many rollouts yield little useful gradient signal. For instance, in standard GRPO training, prompts that are too easy or too hard produce small group-relative advantages, resulting in weak learning signals and wasted rollout budgets~\citep{bae2026onlinedifficulty,li2025knapsackrl,zhang2025speedrl}.

This inefficiency has motivated a growing body of work on \emph{difficulty-aware data selection for RL}, which prioritizes prompts of moderate difficulty where the expected gradient signal is maximized~\citep{yu2025dapo,sun2025dots,zeng2026cures,bae2026onlinedifficulty}. Central to these approaches is the accuracy of the difficulty estimator. In this work, we first conduct a comprehensive analysis of existing difficulty estimation methods~\citep{sun2025dots,pang2026edco,qu2025promptdifficulty,tabib2025toward} for RL and identify three key limitations: \ding{182}~\textit{Inaccurate difficulty estimation under policy drift.} As the policy evolves during RL training, static or slowly adapting estimators become increasingly unreliable. Prompts deemed ``moderately difficult'' may in fact be trivially easy or intractably hard for the current policy, undermining the intended efficiency gains. \ding{183}~\textit{Limited final-performance gains from data selection alone.} Given a sufficient training budget, difficulty-aware data selection alone converges to similar final performance as standard training, with many hard tasks remaining unsolved~\citep{zheng2025actonlywhen,zhang2025speedrl,zeng2026cures}. \ding{184}~\textit{No improvement in inference efficiency.} Models trained with difficulty-based filtering still produce chain-of-thought (CoT) responses of similar length compared to those trained without filtering. These findings suggest that improving both the inference efficiency and effectiveness of RL requires more than selecting moderately difficult data: the policy must also learn to solve hard tasks while preserving its ability to solve easy ones efficiently.

To address these limitations, we propose \ours (\textbf{D}ifficulty-\textbf{A}daptive \textbf{R}einforcement learning with co-\textbf{E}volved difficulty estimation), a unified framework that couples accurate, policy-aligned difficulty estimation with difficulty-adaptive training strategies for efficient and effective RL. First, we introduce a difficulty estimation and policy co-evolution mechanism: by combining self-normalized importance sampling with a prompt-wise First-In-First-Out (FIFO) replay buffer, our method tracks policy changes and maintains accurate difficulty estimates throughout training. We further propose a symmetric Beta sampling distribution that biases selection toward medium-difficulty samples while retaining easy and hard ones, reducing the risk of forgetting easy skills and continuing to expose the policy to hard prompts. Finally, we partition prompts into three difficulty tiers---easy, medium, and hard---and adaptively allocate compute and apply a tailored training strategy to each tier. 

We evaluate \ours across multiple models and task domains. Results show consistent improvements in training efficiency (measured by steps and wall-clock time), final effectiveness, and inference efficiency (measured by average token usage) compared to existing difficulty-aware RL methods. Extensive ablation studies and additional analyses confirm that our method produces shorter and more concise answers on easy tasks while improving correctness on hard tasks. These findings demonstrate that \ours effectively improves both RL training and inference efficiency while adaptively distributing reasoning effort across difficulty levels, providing a promising direction for efficient and effective RL.

Overall, our contributions are fourfold:
\begin{itemize}[leftmargin=1.5em,itemsep=1pt,topsep=2pt,parsep=1pt]
\item[\circmarklight{dareblue}{1}] \textbf{Analysis of existing difficulty-aware data selection for RL.} We identify three empirical gaps in current methods (\autoref{sec:preliminary}): inaccurate difficulty estimation under policy drift, limited final-performance gains from data selection alone, and no improvement in inference efficiency.

\item[\circmarklight{dareblue}{2}] \textbf{Accurate difficulty estimation via policy co-evolution.} We propose a difficulty estimation and policy co-evolution strategy that pairs self-normalized importance sampling with a Beta sampling distribution, yielding more accurate difficulty estimates and diverse difficulty-level data selection to improve training efficiency.

\item[\circmarklight{dareblue}{3}] \textbf{Difficulty-adaptive training strategies.} We combine difficulty-aware data selection with tailored RL training strategies that encourage shorter and more concise solutions for easy tasks while increasing correctness and reasoning depth for hard tasks, enabling dynamic training and inference compute allocation across difficulty levels.

\item[\circmarklight{dareblue}{4}] \textbf{Comprehensive empirical validation.} We conduct extensive experiments across diverse model architectures and tasks. Our method consistently achieves better training efficiency, inference efficiency, and final effectiveness compared to existing difficulty-aware RL methods. Detailed ablation studies further demonstrate the contribution of each component in our design.
\end{itemize}

%% file: sec/3_preliminary.tex
\section{Preliminary Investigation}
\label{sec:preliminary}
Accurate difficulty estimation is the foundation of the difficulty-aware RL method. In this section, we analyze the accuracy of existing difficulty estimation methods and their impact on training efficiency, inference efficiency, and final effectiveness when used for RL training.

\paragraph{Experimental Setting.} We train Qwen2.5-Math-1.5B~\citep{yang2024qwen2} with GRPO under several difficulty-based online data selection methods, and evaluate on MATH-500~\citep{lightman2023let}. The training dataset composition and experimental setup follow \autoref{sec:exp_setup}, with full hyperparameters provided in \appref{sec:preliminary_investigation_details}. All difficulty estimation methods use the same sampling rule, which targets samples with predicted difficulty $0.5$. We compare the following methods:
(i)~\textit{Random Selection}: selects samples uniformly without difficulty consideration;
(ii)~\textit{Embedding-Based Prediction}~\citep{sun2025dots}: computes difficulty based on embedding similarity of a prompt to a representative reference set;
(iii)~\textit{Entropy-Based Estimation}~\citep{pang2026edco}: uses sequence-level entropy as a proxy for difficulty;
(iv)~\textit{LLM-Judge}~\citep{tabib2025toward}: uses the policy model itself as a difficulty judge;
(v)~\textit{Bayesian Posterior Estimation}~\citep{qu2025promptdifficulty}: estimates difficulty via Bayesian posterior updates;
(vi)~\textit{Previous Failure Rate Estimation (Previous FR Estimation)}: estimates difficulty as the failure rate from the latest previous rollouts, computed as one minus the average binary reward over the GRPO group of size $G$;
(vii)~\textit{Current Failure Rate Estimation (Current FR Estimation)}: serves as the \textbf{ground-truth} difficulty reference, computed as the current policy's failure rate over $G$ rollouts. (viii) \textit{Our Estimation}: uses the replay buffer with policy-aligned importance-sampling correction indicated in \autoref{sec:is_estimation}.
Detailed descriptions of all baselines are provided in \autoref{sec:difficulty_estimation_baselines}.

We evaluate estimation accuracy against \textit{Current FR Estimation} using MSE and MAE, and estimate its intrinsic error from the discrepancy between two independent current-policy rollout runs. We also analyze the ground-truth difficulty distribution of samples selected by each method at predicted difficulty $0.5$, measuring how well predicted difficulty matches actual prompt difficulty. Training efficiency is measured by the steps and wall-clock time needed to reach matched performance, while inference efficiency is measured by average output tokens, both overall and stratified by the dataset's original difficulty labels.

% \subsection{Difficulty Estimation Analysis Under Policy Drift}

\begin{figure}[t]
\vspace{-2mm}
\centering
\includegraphics[width=\textwidth]{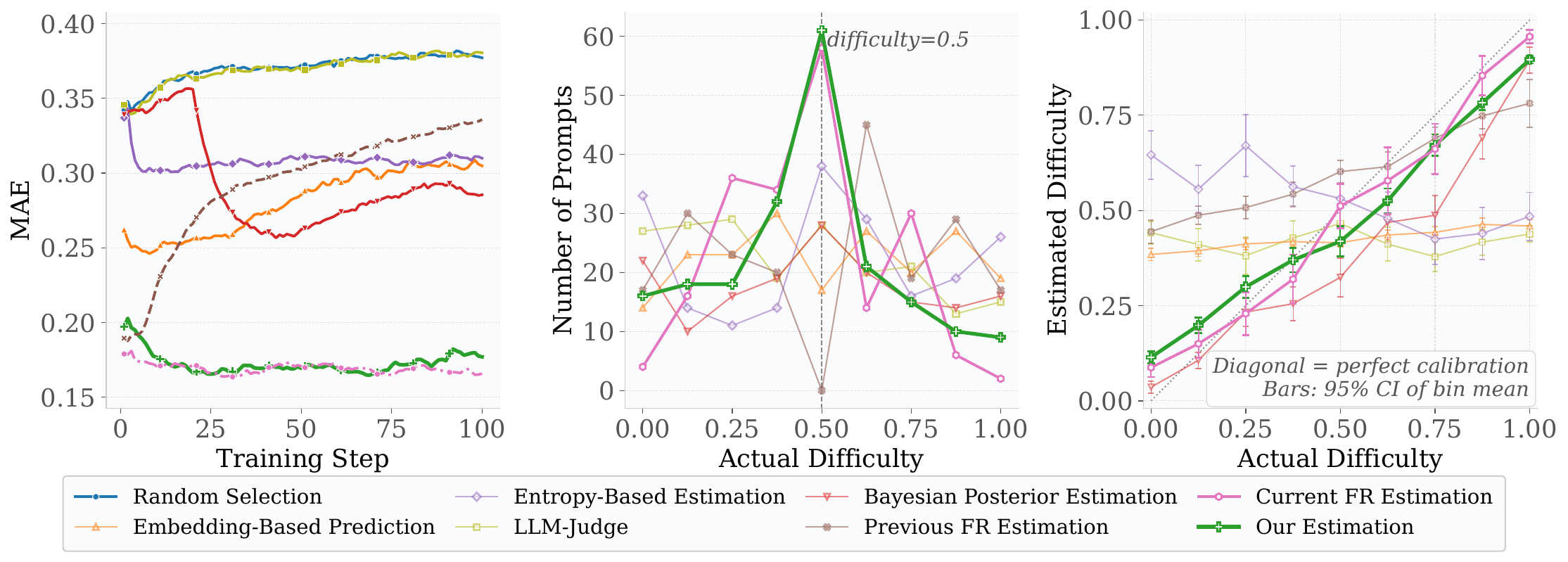}
\caption{Difficulty estimation results for Qwen2.5-Math-1.5B during and after GRPO training under different difficulty-estimation methods. (Left) MAE between each estimator and \textit{Current FR Estimation} across training steps. (Middle) Distribution of ground-truth difficulty for prompts assigned a predicted difficulty of $0.5$ by each method. (Right) Calibration of estimated difficulty against ground-truth difficulty on MATH-500, with bin means and $95\%$ confidence intervals. \textbf{Existing difficulty estimation methods are inaccurate during RL training.}}
\label{fig:difficulty_estimation_analysis}
\vspace{-2mm}
\end{figure}

\paragraph{Difficulty Estimation Analysis Under Policy Drift.} The difficulty estimation accuracy and distribution results are shown in \autoref{fig:difficulty_estimation_analysis} and \autoref{tab:difficulty_accuracy}. We make the following observations: \ding{182}~\textbf{Existing difficulty estimation methods are inaccurate during RL training.} The intrinsic estimation error of \textit{Current FR Estimation} is approximately $0.17$, while the MAE of all existing methods exceeds $0.25$---corresponding to an error of roughly $2$ out of $8$ rollouts, far above the intrinsic baseline. \ding{183}~\textbf{Estimation error grows as training progresses.} The MAE of all existing baselines begins increasing after roughly $40$ training steps, indicating that their difficulty estimates become less reliable as the policy drifts during RL training. \ding{184}~\textbf{Estimated difficulty distributions deviate substantially from ground truth.} When filtering for samples with predicted difficulty $0.5$, all existing methods retain many samples whose ground-truth difficulty is far from $0.5$. Moreover, methods such as \textit{Entropy-Based Estimation} and \textit{LLM-Judge} systematically overestimate difficulty for easy prompts, rather than producing a monotonic mapping from true difficulty to estimated difficulty. These observations show that neither static estimators (e.g., \textit{Previous FR Estimation}) nor existing policy-aware methods (e.g., \textit{Entropy-Based Estimation}) produce reliable difficulty estimates throughout RL training, motivating us to seek an estimator that better tracks policy-dependent difficulty as training evolves.

% \subsection{Training and Inference Efficiency and Effectiveness Analysis}

\begin{figure}[t]
\vspace{-2mm}
\centering
\includegraphics[width=1.0\textwidth]{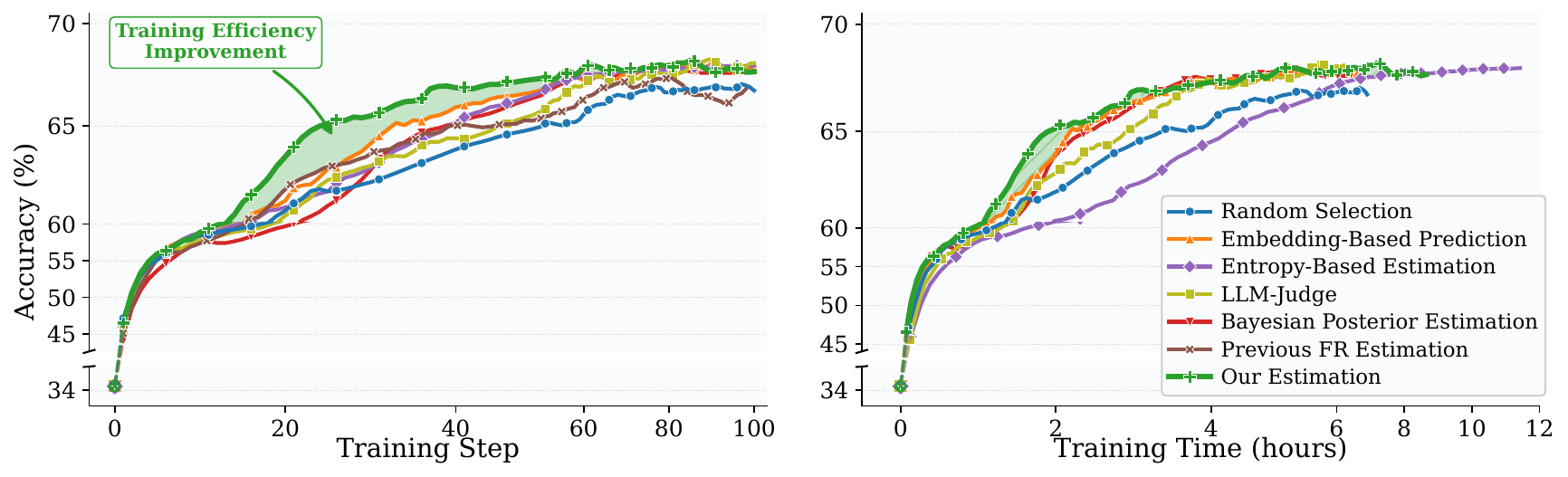}
\caption{Training-step and wall-clock efficiency on MATH-500 with Qwen2.5-Math-1.5B trained by GRPO using different difficulty-estimation methods for data selection. \textbf{Data filtration improves training efficiency but yields only limited final-accuracy gains.}}
\label{fig:training_curves}
\vspace{-2mm}
\end{figure}

\begin{table}[t]
\centering
\caption{Average output token usage and accuracy by the original MATH-500 difficulty level for GRPO using Qwen2.5-Math-1.5B with different difficulty-estimation-based data selection methods. \textbf{Filtration alone does not improve inference efficiency.}}
\label{tab:prelim_inference_tokens}
\resizebox{\textwidth}{!}{%
\begin{tabular}{l c c c c c c}
\toprule
\multirow{2}{*}{Method} & Level 1 & Level 2 & Level 3 & Level 4 & Level 5 & Overall \\
 & \multicolumn{1}{c}{\small Tokens | Acc.} & \multicolumn{1}{c}{\small Tokens | Acc.} & \multicolumn{1}{c}{\small Tokens | Acc.} & \multicolumn{1}{c}{\small Tokens | Acc.} & \multicolumn{1}{c}{\small Tokens | Acc.} & \multicolumn{1}{c}{\small Tokens | Acc.} \\
\midrule
\textit{Random Selection} & 357.3 | 88.4 & 441.5 | 80.0 & 539.4 | 77.1 & 726.2 | 66.4 & 811.6 | 45.5 & 626.9 | 67.4 \\
\textit{Embedding-Based Pred.} & 349.5 | 90.7 & 451.5 | 81.1 & 583.4 | 80.0 & 706.0 | 71.1 & 876.7 | 44.8 & 649.5 | 69.4 \\
\textit{Entropy-Based Est.} & 325.5 | 88.4 & 452.0 | 78.9 & 573.7 | 78.1 & 682.4 | 69.5 & 912.9 | 48.5 & 649.2 | 69.0 \\
\textit{LLM-Judge} & 342.2 | 89.7 & 413.0 | 81.2 & 536.0 | 80.0 & 751.3 | 68.5 & 809.9 | 49.0 & 625.7 | 69.8 \\
\textit{Bayesian Posterior Est.} & 339.2 | 91.2 & 449.8 | 81.4 & 587.4 | 80.2 & 714.5 | 70.9 & 815.2 | 45.6 & 634.8 | 69.7 \\
\textit{Previous FR Est.} & 357.3 | 87.4 & 435.5 | 81.0 & 567.4 | 76.4 & 712.9 | 67.4 & 800.6 | 46.8 & 625.3 | 67.9 \\
% \textit{Current FR Est.} & 350.0 | 83.7 & 446.9 | 82.2 & 573.0 | 83.8 & 829.1 | 62.5 & 857.2 | 49.3 & 672.9 | 68.8 \\
\textit{Our Estimation} & 359.4 | 90.7 & 438.3 | 82.2 & 545.7 | 82.9 & 679.6 | 71.1 & 873.0 | 43.3 & 632.3 | 69.8 \\
\bottomrule
\end{tabular}%
}
\vspace{-2mm}
\end{table}

\paragraph{Training and Inference Efficiency and Effectiveness Analysis.}
We further analyze RL training and inference performance when each estimation method selects only prompts with estimated difficulty $0.5$. Training curves are shown in \autoref{fig:training_curves} and inference results in \autoref{tab:prelim_inference_tokens} (comprehensive evaluation in \autoref{sec:comprehensive_preliminary_results}). We draw the following conclusions: \ding{182}~\textbf{Difficulty filtration improves sample efficiency but yields limited final accuracy gains.} Most difficulty-aware methods converge faster than \textit{Random Selection} GRPO in both training steps and wall-clock time, yet achieve only limited final performance improvements. This suggests that difficulty-aware selection primarily accelerates optimization; improving final accuracy requires training mechanisms beyond sample filtration alone. \ding{183}~\textbf{More accurate difficulty estimation improves training-step efficiency.} \textit{Embedding-Based Prediction} and \textit{Bayesian Posterior Estimation} produce more accurate difficulty estimates than \textit{LLM-Judge} and \textit{Random Selection} (\autoref{fig:difficulty_estimation_analysis}), and exhibit correspondingly faster training-step convergence, confirming that estimation quality directly impacts training efficiency. \ding{184}~\textbf{Filtration alone does not improve inference efficiency.} As shown in \autoref{tab:prelim_inference_tokens}, difficulty-aware methods produce similar token usage to \textit{Random Selection} GRPO across difficulty levels, despite differences in training efficiency. This is expected because filtration changes which prompts are sampled during training, but does not teach the model to allocate different reasoning lengths to different difficulty levels. These findings motivate difficulty-adaptive training strategies that go beyond data filtration to improve RL training efficiency, inference efficiency, and final effectiveness.

%% file: sec/4_method.tex
\section{Method}
\label{sec:method}

\begin{figure}[t]
\vspace{-2mm}
\centering
\includegraphics[width=\textwidth]{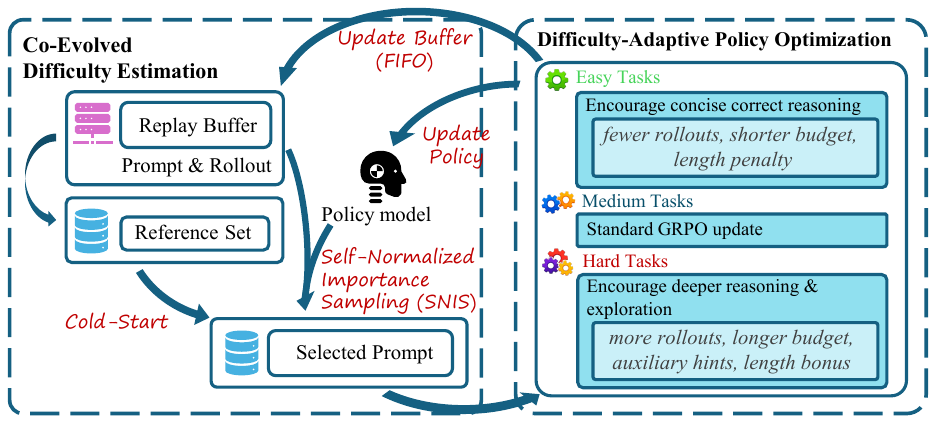}
\caption{Overview of \ours. At each epoch, \ours updates a prompt-wise replay buffer, estimates current prompt difficulty with SNIS, and uses the estimate for dynamic prompt selection and difficulty-adaptive policy optimization.}
\label{fig:overview}
\vspace{-2mm}
\end{figure}

The analysis in \autoref{sec:preliminary} shows that existing difficulty-aware RL suffers from inaccurate estimates under policy drift, while data filtration alone yields limited final-accuracy gains and does not improve inference efficiency. To address these issues, we propose \ours, a difficulty-adaptive RL framework that integrates policy-aligned difficulty estimation, dynamic data selection, and difficulty-adaptive policy optimization. At each training epoch, \ours estimates current prompt difficulty from a replay buffer with policy-aligned correction, samples prompts with a bias toward medium difficulty while retaining easy and hard prompts, and applies difficulty-specific compute budgets and reward shaping during policy optimization. Unlike filtration-only methods, which use difficulty estimates only for data selection, \ours adapts both the training data and learning signals as the policy evolves. An overview is shown in \autoref{fig:overview}, and pseudo-code is provided in \autoref{sec:algorithm}.

\subsection{Co-Evolved Difficulty Estimation}
\label{sec:is_estimation}

\paragraph{Replay buffer and cold start.}
We maintain a FIFO replay buffer $\mathcal{B}$ to support difficulty estimation and sample reuse. For each prompt $q$, let \(K_q\) be the number of currently buffered rollouts; when \(q\) is fixed, we write \(K=K_q\). The buffer stores
\begin{equation}
\label{eq:replay_buffer}
    \mathcal{B}_q
    =
    \left\{
    \left(
    o_k,
    r_k,
    \log \pi_{\theta_{\mathrm{beh},k}}(o_k\mid q)
    \right)
    \right\}_{k=1}^{K_q},
\end{equation}
where $o_k$ is a generated response, $r_k\in\{0,1\}$ is the binary outcome reward, and $\pi_{\theta_{\mathrm{beh},k}}$ is the behavior policy that generated $o_k$. 

For prompts without buffered rollouts, we use an embedding-based cold start \citep{sun2025dots}. We sample a small reference set $\mathcal{D}_{\mathrm{ref}}=\{p_i\}_{i=1}^{N}$ from the training set $\mathcal{D}$, where $|\mathcal{D}_{\mathrm{ref}}|\ll|\mathcal{D}|$. For each reference prompt $p_i$, we generate $G$ rollouts and define its initial difficulty as the failure rate,
\begin{equation}
    d_{p_i}
    =
    1-\bar r_{p_i}
    =
    \frac{1}{G}\sum_{k=1}^{G}(1-r_k).
    \label{eq:success_rate_difficulty}
\end{equation}
For an unseen prompt $q$, we compute a similarity-weighted difficulty estimate,
\begin{equation}
    a_i
    =
    \frac{\exp(z_q^\top z_i/\sqrt{h})}
    {\sum_{j=1}^{N}\exp(z_q^\top z_j/\sqrt{h})},
    \qquad
    \hat d_q
    =
    \sum_{i=1}^{N} a_i d_{p_i},
    \label{eq:embedding_weight_sum}
\end{equation}
where $z_q$ and $z_i$ are prompt embeddings and $h$ is the embedding dimension. This provides an initial difficulty estimate before sufficient rollout history is available.

\paragraph{Policy-aligned replay estimation.}
Once a prompt has buffered rollouts, directly averaging historical rewards can be biased because those rollouts may come from older policies. We correct this policy drift with self-normalized importance sampling (SNIS). For each buffered rollout $o_k$, we compute the current-to-behavior log-ratio and clipped weight as
\begin{equation}
    w_k
    =
    \exp\Biggl(
    \operatorname{clip}
    \Bigl(
    \sum_{t=1}^{|o_k|}
    \bigl[
    \log \pi_\theta(o_{k,t}\mid q,o_{k,<t})
    -
    \log \pi_{\theta_{\mathrm{beh},k}}(o_{k,t}\mid q,o_{k,<t})
    \bigr],
    -c,
    c
    \Bigr)
    \Biggr),
    \label{eq:is_weights}
\end{equation}
where clipping is applied for numerical stability. Here \(\pi_\theta\) and \(\pi_{\theta_{\mathrm{beh},k}}\) denote the actual token-level rollout distributions after applying the decoding and stopping rules, rather than raw model softmax distributions. The current-policy difficulty estimate is then
$\hat d_q
=
\frac{
\sum_{k=1}^{K} w_k(1-r_k)
}{
\sum_{k=1}^{K} w_k
}.$
This estimator tracks policy evolution by upweighting rollouts likely under the current policy and downweighting stale ones. In practice, it avoids the cost of re-rolling every prompt at each step, as required by \textit{Current FR Estimation}, while keeping replay statistics policy-aligned and achieving comparable estimation accuracy, as shown in \autoref{sec:preliminary}.

We further give an idealized finite-sample decomposition for the SNIS estimate. Fix a prompt \(q\), an estimation epoch, and its \(K\) buffered rollouts. Let \(Y(o)=1-r(o)\), \(d_\theta(q)=\mathbb{E}_{O\sim\pi_\theta(\cdot\mid q)}[Y(O)]\), and \(\rho_k(o)=\frac{d\pi_\theta(\cdot\mid q)}{d\mu_k(\cdot\mid q)}(o)\) for \(O_k\sim\mu_k(\cdot\mid q)\), where \(\pi_\theta\) and \(\mu_k\) denote the actual current and behavior rollout distributions, including decoding and stopping. Let \(W_k\) be the clipped weight with \(0\le W_k(o)\le C_w\); for \autoref{eq:is_weights}, \(C_w=e^c\).

\begin{proposition}[Finite-sample error of clipped SNIS difficulty estimation, the Proof is detailed in \autoref{sec:app_snis_finite_sample_proof}]
\label{prop:snis_finite_sample}
Suppose \(O_1,\ldots,O_K\) are conditionally independent given the fixed epoch and current policy, \(\pi_\theta(\cdot\mid q)\ll\mu_k(\cdot\mid q)\) for all \(k\), and \(K^{-1}\sum_k\mathbb{E}_k[W_k(O)]\ge b_{\min}>0\), where \(\mathbb{E}_k\) denotes expectation under \(O\sim\mu_k(\cdot\mid q)\). For \(\hat d_q=(K^{-1}\sum_k W_k(O_k)Y(O_k))/(K^{-1}\sum_k W_k(O_k))\) and \(\varepsilon_\delta=C_w\sqrt{\log(4/\delta)/(2K)}<b_{\min}\), with probability at least \(1-\delta\),
\begin{equation}
\label{eq:finite_sample_bound}
\left|\hat d_q-d_\theta(q)\right|
\le
\frac{2\varepsilon_\delta}{b_{\min}-\varepsilon_\delta}
+
\frac{1}{b_{\min}K}
\sum_{k=1}^{K}
\mathbb{E}_k\!\left[\left|\rho_k(O)-W_k(O)\right|\right].
\end{equation}
For \autoref{eq:is_weights}, the second term equals \((b_{\min}K)^{-1}\sum_k\mathbb{E}_k[(\rho_k(O)-e^c)_+ +(e^{-c}-\rho_k(O))_+]\). The worst-case concentration term can be loose for large \(c\), so we use ESS (see \appref{sec:algorithm}) in practice.
\end{proposition}

\subsection{Difficulty-Guided Dynamic Data Selection}
\label{sec:sample_selection}
Given the estimated difficulty $\hat d_q$, \ours samples prompts for each training batch. Since $\hat d_q$ estimates the failure rate, the expected binary reward is approximately $1-\hat d_q$, with variance $\hat d_q(1-\hat d_q)$. This variance is maximized at $\hat d_q=0.5$, where group-relative advantages provide the strongest learning signal. We therefore bias sampling toward medium-difficulty prompts using a symmetric Beta density,
\begin{equation}
    p_{\ours}(q)
    \propto
    \mathrm{Beta}(\hat d_q;\alpha,\alpha),
    % \qquad
    \alpha=1+\kappa/2,
    \label{eq:beta_sampling}
\end{equation}
where $\kappa$ controls the concentration around $\hat d_q=0.5$.

Unlike hard-threshold filtering that retains only prompts near a fixed difficulty value~\citep{sun2025dots}, \autoref{eq:beta_sampling} assigns nonzero probability across the full difficulty range. Easy and hard prompts remain in training but appear less frequently than medium ones, reducing the risk of forgetting easy skills while continuing to expose the policy to hard prompts~\citep{parashar2026curriculum,kordi2026revisiting,yang2024can}. Within each batch, prompts are sampled without replacement according to $p_{\ours}(q)$.

\subsection{Difficulty-Adaptive Policy Optimization}
\label{sec:stratified_rl}

Data selection determines \emph{which} prompts are trained; difficulty-adaptive policy optimization determines \emph{how} each is trained. We partition selected prompts into three tiers using two thresholds $d_{\mathrm{easy}}$ and $d_{\mathrm{hard}}$: easy ($\hat d_q<d_{\mathrm{easy}}$), medium ($d_{\mathrm{easy}}\leq \hat d_q\leq d_{\mathrm{hard}}$), and hard ($\hat d_q>d_{\mathrm{hard}}$). Easy prompts are trained to produce concise correct reasoning, medium prompts receive standard GRPO updates, and hard prompts are allocated additional exploration compute to obtain useful trajectories.

We optimize the policy with the following difficulty-conditioned objective:
\begin{equation}
\label{eq:ours_grpo}
\begin{aligned}
\mathcal{J}_{\ours}(\theta)&=\mathbb{E}_{q\sim p_{\ours},\{o_i\}_{i=1}^{G_q}\sim\pi_{\theta_{\mathrm{beh}}}(\cdot\mid q)}\Biggl[
\frac{1}{G_q}\sum_{i=1}^{G_q}\frac{1}{|o_i|}\sum_{t=1}^{|o_i|}
\Biggl(\min\biggl(
\frac{\pi_\theta(o_{i,t}\mid q,o_{i,<t})}
{\pi_{\theta_{\mathrm{beh}}}(o_{i,t}\mid q,o_{i,<t})}A_i,
\\
&
\operatorname{clip}(
\frac{\pi_\theta(o_{i,t}\mid q,o_{i,<t})}
{\pi_{\theta_{\mathrm{beh}}}(o_{i,t}\mid q,o_{i,<t})},
1-\epsilon_q^{-},1+\epsilon_q^{+}
)A_i
\biggr)-\beta\,\mathrm{D}_{\mathrm{KL}}\!\left(\pi_\theta\Vert\pi_{\mathrm{ref}}\right)
\Biggr)
\Biggr].
\end{aligned}
\end{equation}
Here, $p_{\ours}$ is the sampling distribution from \autoref{eq:beta_sampling}, $G_q$ is the rollout group size for prompt $q$, $A_i$ is the group-relative advantage computed from the shaped reward $\tilde r_i$, and $\epsilon_q^{-},\epsilon_q^{+}$ are difficulty-conditioned clipping ranges. Each batch mixes a fraction $\sigma$ of fresh on-policy rollouts with $1-\sigma$ replay buffer trajectories, where $\pi_{\theta_{\mathrm{beh}}}$ denotes the behavior policy that generated replayed trajectories.

\paragraph{\iconeasybig Concise reasoning for easy prompts.}
Easy prompts typically yield uniformly correct rollouts with little advantage variation. \ours reduces the rollout count to $G_q = G_{\mathrm{easy}} < G$, imposes a shorter response budget, and applies a length-dependent penalty to correct rollouts to encourage concise solutions. We define an easiness weight $w_q^{\mathrm{easy}} = \operatorname{clamp}\bigl((d_{\mathrm{easy}} - \hat d_q)/d_{\mathrm{easy}},\, 0,\, 1\bigr)$, which increases as the prompt becomes easier for the current policy. The shaped reward is
\begin{equation}
    \tilde r_i
    =
    r_i
    -
    \mathbf{1}_{\{\hat d_q<d_{\mathrm{easy}}\}}
    \mathbf{1}_{\{r_i=1\}}\,
    \lambda_{\mathrm{easy}}\,
    w_q^{\mathrm{easy}}\,
    \frac{|o_i|}{T_{\max}},
    \label{eq:easy_shaped_reward}
\end{equation}
where $\lambda_{\mathrm{easy}}$ is the penalty coefficient and $T_{\max}$ is the maximum rollout length within the group. This gives higher advantages to shorter correct solutions while leaving incorrect rollouts unchanged, teaching the policy to save inference tokens on prompts it already solves reliably. We also relax the upper clipping bound ($\epsilon_q^{+} > \epsilon$) while keeping $\epsilon_q^{-} = \epsilon$, allowing stronger positive updates toward concise correct solutions without amplifying harmful updates.

\paragraph{\iconmediumbig Standard updates for medium prompts.}
Medium prompts use the base configuration: $G_q = G$, $\tilde r_i = r_i$, and $\epsilon_q^{-} = \epsilon_q^{+} = \epsilon$. Their rollout groups typically contain both correct and incorrect responses, providing sufficient advantage variation without additional shaping. The Beta sampler in \autoref{sec:sample_selection} ensures that medium prompts appear most frequently in each batch, making this tier the backbone of optimization.

\paragraph{\iconhardbig Exploration incentive for hard prompts.}
Hard prompts often produce rollout groups in which every response is incorrect, collapsing the group-relative advantage. \ours addresses this by allocating more rollouts ($G_q = G_{\mathrm{hard}} > G$) to increase the probability of obtaining at least one correct trajectory. Specifically, \ours generates $G$ standard rollouts for prompt $q$ alongside $G_q - G$ hint-augmented rollouts, where the hint is a successful historical trajectory retrieved from the replay buffer (if no successful trajectory exists, all $G_q$ rollouts use the standard prompt). To further support exploration, a bounded length bonus is added to incorrect rollouts. Analogously to the easy case, we define a hardness weight $w_q^{\mathrm{hard}} = \operatorname{clamp}\bigl((\hat d_q - d_{\mathrm{hard}})/(1 - d_{\mathrm{hard}}),\, 0,\, 1\bigr)$ and shape the reward as
\begin{equation}
    \tilde r_i
    =
    r_i
    +
    \mathbf{1}_{\{\hat d_q>d_{\mathrm{hard}}\}}
    \mathbf{1}_{\{r_i=0\}}\,
    \lambda_{\mathrm{hard}}\,
    w_q^{\mathrm{hard}}\,
    \frac{|o_i|}{T_{\max}}.
    \label{eq:hard_shaped_reward}
\end{equation}
This bonus encourages the policy to invest more reasoning effort on hard prompts rather than giving up early. Since $\lambda_{\mathrm{hard}} < 1$, any correct rollout still receives a higher reward than an incorrect one with the bonus, ensuring that correctness remains the dominant training signal.

Together, the easy-prompt penalty and hard-prompt bonus create a difficulty-adaptive length incentive: the policy learns to be concise on prompts it solves reliably and to reason more deeply on prompts where it struggles. This behavior cannot arise from data filtering alone, which changes \emph{which} prompts are trained but does not shape \emph{how much} reasoning effort each difficulty level receives.

%% file: sec/5_experiment.tex
\section{Experiments}
\label{sec:experiments}
We evaluate \ours along the three axes established in \autoref{sec:preliminary}: training efficiency, final accuracy, and inference-time token usage. \autoref{sec:main_results} reports the main results across five benchmarks, while \autoref{sec:ablation} analyzes the contribution of individual designs and the effects of hyperparameter tuning.

\subsection{Experimental Setup}
\label{sec:exp_setup}

\paragraph{Models and training data.}
We evaluate on three model scales: Qwen2.5-Math-1.5B, Qwen2.5-Math-7B~\citep{yang2024qwen2}, and SmolLM3-3B-Base~\citep{bakouch2025smollm3}, covering both math-specialized and general-purpose architectures. All models are trained with GRPO~\citep{shao2024deepseekmath} on the same union composed of MATH Level~3--5~\citep{hendrycks2021measuring}, DeepScaleR-40K~\citep{luo2025deepscaler}, Open-Reasoner-Zero-57K~\citep{hu2025openreasonerzero}, and DeepMath-103K~\citep{he2025deepmath} ($35{,}512$ prompts in total). All prompts share a uniform system prompt, with difficulty-aware hints appended to the user content for hard prompts; prompt templates, and dataset details are provided in \appref{sec:prompt_templates} and \appref{sec:preliminary_investigation_details}.

\paragraph{Implementation details.}
Unless otherwise stated, we use the AdamW optimizer, base rollout count $G{=}8$, maximum response length $T_{\max}{=}3072$ tokens, sampling temperature $0.6$, and top-$p$ $0.95$. For \ours, the cold-start reference size is $N{=}256$, the SNIS log-ratio clip is $c{=}4.0$, and the Beta concentration is $\kappa{=}100$ in \autoref{eq:beta_sampling}. Prompts are stratified at $(d_{\mathrm{easy}},d_{\mathrm{hard}}){=}(0.3,0.8)$ with tiered rollout counts $(G_{\mathrm{easy}},G_{\mathrm{hard}}){=}(4,16)$, reward-shaping coefficients $\lambda_{\mathrm{easy}}{=}\lambda_{\mathrm{hard}}{=}1\times10^{-4}$, and symmetric clipping $\epsilon{=}0.2$ on medium and hard prompts; for easy prompts we use the relaxed upper bound selected by the ablation in \autoref{tab:clip_sweep}. The per-prompt rollout history size $K$ and replay buffer capacity are reported in \appref{sec:preliminary_investigation_details}.

\paragraph{Baselines.}
We compare against six difficulty-aware baselines. Unless otherwise specified, each baseline follows the training pipeline of its source paper, unlike in \autoref{sec:preliminary}, where all baselines use the same data-sampling rule (i)~\textit{GRPO}~\citep{shao2024deepseekmath}: uniform prompt sampling without difficulty-aware selection. (ii)~\textit{DOTS}~\citep{sun2025dots}: Embedding-Based Prediction with rollout replay. (iii)~\textit{EDCO}~\citep{pang2026edco}: Entropy-Based Estimation with dynamic curriculum selection. (iv)~\textit{MoPPS}~\citep{qu2025promptdifficulty}: Bayesian Posterior Estimation with Thompson-sampling prompt selection. (v)~\textit{LLM-Judge} and (vi)~\textit{Previous FR Estimation} reuse the protocol of \autoref{sec:preliminary}, where GRPO is performed on prompts with estimated difficulty $0.5$.

\paragraph{Evaluation.}
We evaluate on five math benchmarks: MATH-500~\citep{lightman2023let}, GSM8K~\citep{cobbe2021training}, AIME-AMC~\citep{aimo2024aime,aimo2024amc}, MinervaMath~\citep{lewkowycz2022minerva}, and OlympiadBench~\citep{he2024olympiadbench}. For the reward-shaping and clipping ablations (\autoref{tab:el_hl_sweep}, \autoref{tab:clip_sweep}), we also report \emph{AIME24/25}~\citep{maa2024aime}, a separate $60$-problem set combining AIME~2024 and AIME~2025. To test whether the gains transfer beyond mathematics, we evaluate on three coding benchmarks: HumanEval~\citep{chen2021humaneval}, MBPP~\citep{austin2021mbpp}, and LiveCodeBench~\citep{jain2025livecodebench_iclr}, with the corresponding training data and protocol detailed in \appref{sec:coder_transfer}. Unless stated otherwise, all results use greedy decoding.

\subsection{Main Results}
\label{sec:main_results}

\begin{figure}[t]
% \vspace{-2mm}
\centering
\includegraphics[width=1.0\textwidth]{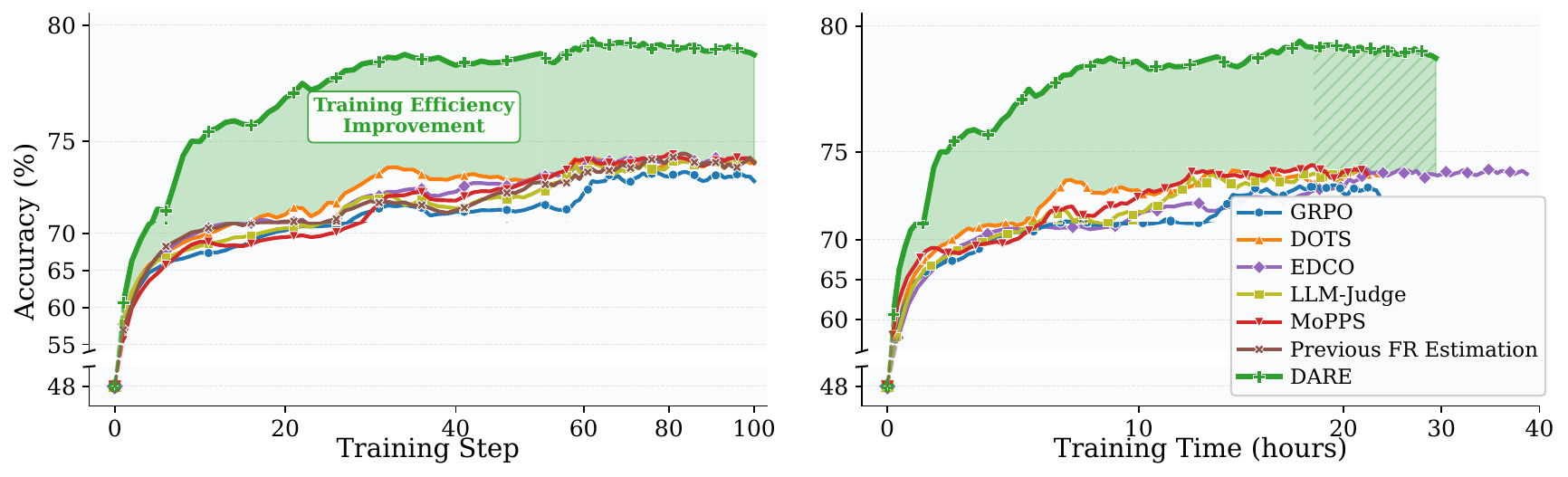}
\caption{Training performance on MATH-500 (left) and GSM8K (right) with Qwen2.5-Math-7B under various difficulty-aware RL methods. \ours improves both training efficiency and effectiveness.}
\label{fig:training_curves_7b}
\end{figure}

\begin{table}[t]
% \vspace{-2mm}
\centering
\caption{Accuracy (\%) with different difficulty-aware RL methods on Qwen2.5-Math-1.5B. \textcolor{green!50!black}{\small$\blacktriangle$} / \textcolor{red!70!black}{\small$\blacktriangledown$} denotes accuracy improvement / degradation relative to \textit{GRPO}. }
\label{tab:main_results_1p5b}
\resizebox{0.9\textwidth}{!}{%
\begin{tabular}{l cc cc cc cc}
\toprule
Method &
GSM8K & AIME-AMC &
Minerva & Olympiad \\
\midrule
\textit{GRPO} & 77.63  & 45.78  & 13.24  & 22.55  \\
DOTS & 79.38\,{\scriptsize\up{1.75}}  & 45.78 & 15.44\,{\scriptsize\up{2.20}} & 23.74\,{\scriptsize\up{1.19}}  \\
EDCO & 80.59\,{\scriptsize\up{2.96}}  & 46.99\,{\scriptsize\up{1.21}}  & 13.60\,{\scriptsize\up{0.36}}  & 24.63\,{\scriptsize\up{2.08}}  \\
LLM-Judge & 79.08\,{\scriptsize\up{1.45}}  & 42.17\,{\scriptsize\down{3.61}}  & 13.24  & 24.04\,{\scriptsize\up{1.49}} \\
MoPPS & 82.87\,{\scriptsize\up{5.24}}  & 46.99\,{\scriptsize\up{1.21}} & 16.54\,{\scriptsize\up{3.30}}  & 21.96\,{\scriptsize\down{0.59}} \\
Previous FR Estimation & 76.50\,{\scriptsize\down{1.13}} & 44.58\,{\scriptsize\down{1.20}} & 13.24 & 23.15\,{\scriptsize\up{0.60}} \\
\midrule
% \ours~(EL penalty only) & \textbf{86.43}\,{\scriptsize\up{8.80}} & \textbf{233} & 46.99\,{\scriptsize\up{1.21}} & \textbf{796} & 13.87\,{\scriptsize\up{0.63}} & \textbf{720} & 22.96\,{\scriptsize\up{0.41}} & \textbf{764} \\
% \ours~(EL penalty + HL bonus) & 82.49\,{\scriptsize\up{4.86}} & 351 & 46.99\,{\scriptsize\up{1.21}} & 956 & 17.28\,{\scriptsize\up{4.04}} & 882 & 25.07\,{\scriptsize\up{2.52}} & 1062 \\
\ours & \textbf{84.53}\,{\scriptsize\up{6.90}} & \textbf{50.60}\,{\scriptsize\up{4.82}} & \textbf{19.85}\,{\scriptsize\up{6.61}} & \textbf{26.41}\,{\scriptsize\up{3.86}} \\
\bottomrule
\end{tabular}%
}
% \vspace{-2mm}
\end{table}

\begin{table}[t]
% \vspace{-2mm}
\centering
\caption{Per-difficulty-level average output tokens and accuracy (\%) on MATH-500 with Qwen2.5-Math-1.5B. \textcolor{green!50!black}{\small$\blacktriangle$} / \textcolor{red!70!black}{\small$\blacktriangledown$} denotes improvement / degradation relative to \textit{GRPO}.}
\label{tab:math_500_per_level}
\newcommand{\vsep}{\hskip 4pt{\color{black!30}\vrule width 0.4pt}\hskip 4pt}
\resizebox{\textwidth}{!}{%
\setlength{\tabcolsep}{4pt}
\begin{tabular}{l c@{\hskip 3pt}c @{\vsep} c@{\hskip 3pt}c @{\vsep} c@{\hskip 3pt}c @{\vsep} c@{\hskip 3pt}c @{\vsep} c@{\hskip 3pt}c @{\vsep} c@{\hskip 3pt}c}
\toprule
\multirow{2}{*}{Method} &
\multicolumn{2}{c}{Level 1} & \multicolumn{2}{c}{Level 2} &
\multicolumn{2}{c}{Level 3} & \multicolumn{2}{c}{Level 4} &
\multicolumn{2}{c}{Level 5} & \multicolumn{2}{c}{Overall} \\
\cmidrule(lr){2-3}\cmidrule(lr){4-5}\cmidrule(lr){6-7}\cmidrule(lr){8-9}\cmidrule(lr){10-11}\cmidrule(lr){12-13}
 & Acc & Tok & Acc & Tok & Acc & Tok & Acc & Tok & Acc & Tok & Acc & Tok \\
\midrule
\textit{GRPO} & 88.37 & 357 & 80.00 & 442 & 77.14 & 539 & 66.41 & 726 & 45.52 & 812 & 67.40 & 627 \\
% \midrule
% \multirow{2}{*}{\ours~(EL+HL)} & 88.37 & 271 & 83.33 & 433 & 83.81 & 523 & 71.09 & 715 & 47.01 & 792 & 71.60 & 622 \\
%  & --- & {\scriptsize\up{24\%}} & {\scriptsize\up{3.33}} & {\scriptsize\up{2\%}} & {\scriptsize\up{6.67}} & {\scriptsize\up{3\%}} & {\scriptsize\up{4.68}} & {\scriptsize\up{2\%}} & {\scriptsize\up{1.49}} & {\scriptsize\up{2\%}} & {\scriptsize\up{4.20}} & {\scriptsize\up{1\%}} \\
\midrule
% \multirow{2}{*}{\ours} & 90.37 & \textbf{212} & 84.33 & \textbf{290} & 80.10 & \textbf{365} & 66.84 & \textbf{569} & 46.76 & \textbf{680} & 70.10 & \textbf{475} \\
%  & {\scriptsize\up{2.00}} & {\scriptsize\up{41\%}} & {\scriptsize\up{4.33}} & {\scriptsize\up{34\%}} & {\scriptsize\up{2.96}} & {\scriptsize\up{32\%}} & {\scriptsize\up{0.43}} & {\scriptsize\up{22\%}} & {\scriptsize\up{1.24}} & {\scriptsize\up{16\%}} & {\scriptsize\up{2.70}} & {\scriptsize\up{24\%}} \\
\multirow{2}{*}{\ours} 
& \textbf{87.45} & 356 
& \textbf{84.62} & 428 
& \textbf{84.40} & 518 
& \textbf{70.93} & 689 
& \textbf{53.64} & 912 
& \textbf{72.60} & 601 \\
& {\scriptsize\down{0.92}} & {\scriptsize\up{0\%}} 
& {\scriptsize\up{4.62}} & {\scriptsize\up{3\%}} 
& {\scriptsize\up{7.26}} & {\scriptsize\up{4\%}} 
& {\scriptsize\up{4.52}} & {\scriptsize\up{5\%}} 
& {\scriptsize\up{8.12}} & {\scriptsize\down{12\%}} 
& {\scriptsize\up{5.20}} & {\scriptsize\up{4\%}} \\

\bottomrule
\end{tabular}
}
% \vspace{-2mm}
\end{table}

\paragraph{\ours's difficulty estimation is the most accurate among existing methods.}
As shown in \autoref{fig:difficulty_estimation_analysis} and \autoref{tab:difficulty_accuracy}, our policy-aligned SNIS estimator achieves the lowest error among all difficulty estimation methods during and after RL training, approaching the intrinsic error of \textit{Current FR Estimation}. This result shows that SNIS can effectively correct policy drift in difficulty estimation.

\paragraph{\ours's unified difficulty-adaptive RL framework improves both training efficiency and effectiveness.}
\autoref{fig:training_curves_7b} shows the training performance on MATH-500 and GSM8K with Qwen2.5-Math-7B. Compared with existing difficulty-aware RL methods, \ours achieves higher accuracy with fewer training steps and shorter wall-clock time. This result differs from the filtration-only results in \autoref{fig:training_curves}, where different methods converge to similar final accuracy, showing that the gains of \ours come from combining difficulty-guided dynamic data selection with difficulty-adaptive policy optimization rather than from data filtration alone. This advantage also transfers to SmolLM3-3B-Base, as shown in \appref{sec:ablation_smollm}, where \ours outperforms all baselines on MATH-500 in both training efficiency and final accuracy. Moreover, \autoref{tab:main_results_1p5b} reports final accuracy on four reasoning benchmarks, where \ours consistently achieves the best performance across methods and datasets. The coding transfer results in \autoref{sec:coder_transfer} further show that \ours also improves efficiency and effectiveness on LiveCodeBench, MBPP, and HumanEval. Finally, we compare the peak and average GPU and CPU memory usage of all RL methods. The results in \autoref{tab:prelim_memory_usage} show that \ours improves performance without increasing GPU memory usage.

\paragraph{\ours{} improves inference efficiency and effectiveness simultaneously.}
As shown in \autoref{tab:math_500_per_level}, compared with the GRPO baseline, \ours improves both overall inference efficiency and accuracy. The efficiency gains are most clear on easy prompts, where \ours uses fewer tokens while maintaining high correctness. We further provide a case study in \autoref{sec:case_study}, where \ours and GRPO answer the same MATH-500 question; \ours produces a shorter correct response for an easy prompt. These results show that reward shaping and dynamic compute allocation help the policy allocate reasoning effort more efficiently across difficulty levels.

% \paragraph{Estimator accuracy translates into faster convergence.}
% The convergence gap between IS and the other filtration-only rows in \autoref{fig:training_curves_7b} is consistent with the estimation-error ranking established in \autoref{sec:preliminary}: IS attains MAE $0.189$ in \autoref{tab:difficulty_accuracy}, below \textit{DOTS} ($0.266$), \textit{MoPPS} ($0.254$), \textit{LLM-Judge} ($0.269$), and \textit{Previous FR Estimation} ($0.289$), approaching the intrinsic $0.169$ of \textit{Current FR Estimation}. The ordering of the training curves in \autoref{fig:training_curves_7b} matches this MAE ranking, in line with Finding~5 in \autoref{sec:preliminary}. 

% A budget-level quantification of this convergence gap---including the formal definitions of steps-to-target, step speedup, and normalized $\mathrm{AUC}$, and the corresponding numbers on Qwen2.5-Math-1.5B and Qwen2.5-Math-7B---is reported in \appref{sec:training_efficiency}.

\subsection{Ablation Studies and Additional Investigations}
\label{sec:ablation}

% \paragraph{\ours component design ablation.} We conduct comprehensive ablation studies to the component designs in \ours, specifcally, we investigate the indivisual controbutions of our co-evolved difficulty estimation with self-normalized importance sampling (\textit{SNIS}), difficulty-guided dynamic data selection (\textit{dynamic data selection}), concise reasoning for easy prompts (\textit{easy penalty}), and exploration incentive for hard prompts (\textit{hard bonus}) using Qwen2.5-Math-1.5B, Qwen2.5-Math-7B and SmolLM3-3B-Base evaluted on MATH-500, GSM8K, AIME-AMC, and Minerva. The results shown in \autoref{tab:main_results} and \autoref{tab:ablation_1_5b} in the Appendix indicate the following observations: \ding{182} Each of our component design can effectively enhance the final performance; \ding{183} \ours achieves the best overll performance when all component are utilized, indicating the effect of our component designs are orthogonal.

\paragraph{\ours component-design ablation.}

\begin{table}[t]
% \vspace{-2mm}
\centering
\caption{Ablation of difficulty-adaptive training components on Qwen2.5-Math-7B across four benchmarks. Each row adds one component \emph{independently} on top of the \textit{GRPO} baseline; the bottom row \emph{\ours} combines all components. \textcolor{green!50!black}{\small$\blacktriangle$} / \textcolor{red!70!black}{\small$\blacktriangledown$} denotes improvement / degradation over \textit{GRPO}.}
% \appref{sec:ablation_1.5b} and \appref{sec:ablation_smollm}.}
\label{tab:main_results}
\resizebox{\textwidth}{!}{%
\begin{tabular}{lcccccc}
\toprule
Method & MATH-500 & GSM8K & AIME-AMC & Minerva & Overall  \\
\midrule
\textit{GRPO} & 66.8 & 77.6 & 45.8 & 13.2 & 50.9 \\
\quad + \textit{SNIS} & 75.8\,{\scriptsize\up{9.0}} & 88.7\,{\scriptsize\up{11.1}} & 51.5\,{\scriptsize\up{5.7}} & 16.3\,{\scriptsize\up{3.1}} & 58.1\,{\scriptsize\up{7.2}}\\
\quad + \iconeasy\iconmedium\iconhard \textit{dynamic data selection} & 77.0\,{\scriptsize\up{10.2}} & 91.0\,{\scriptsize\up{13.4}} & 51.8\,{\scriptsize\up{6.0}} & 17.0\,{\scriptsize\up{3.8}} & 59.2\,{\scriptsize\up{8.3}} \\
% \quad + \iconhard Think-longer re-rollout & 77.3\,{\scriptsize\up{10.5}} & 88.9\,{\scriptsize\up{11.3}} & 53.0\,{\scriptsize\up{7.2}} & 17.0\,{\scriptsize\up{3.8}} & 59.1\,{\scriptsize\up{8.2}} \\
% \quad + \iconhard Memory-guided hint & 76.0\,{\scriptsize\up{9.2}} & 89.9\,{\scriptsize\up{12.3}} & 53.0\,{\scriptsize\up{7.2}} & 18.0\,{\scriptsize\up{4.8}} & 59.2\,{\scriptsize\up{8.3}} \\
\quad + \iconeasy \textit{easy penalty} & 74.8\,{\scriptsize\up{8.0}} & \textbf{93.0}\,{\scriptsize\up{15.4}} & 51.8\,{\scriptsize\up{6.0}} & 14.5\,{\scriptsize\up{1.3}} & 58.5\,{\scriptsize\up{7.6}} \\
\quad + \iconhard \textit{hard bonus} & 77.0\,{\scriptsize\up{10.2}} & 90.2\,{\scriptsize\up{12.6}} & 51.8\,{\scriptsize\up{6.0}} & 19.0\,{\scriptsize\up{5.8}} & 59.5\,{\scriptsize\up{8.6}} \\
% \quad + \iconeasy Relaxed upper clipping & 76.2\,{\scriptsize\up{9.4}} & 89.2\,{\scriptsize\up{11.6}} & 54.2\,{\scriptsize\up{8.4}} & \textbf{20.0}\,{\scriptsize\up{6.8}} & 59.9\,{\scriptsize\up{9.0}} \\
\cmidrule(lr){1-6}
\ours (all components) & \textbf{78.2}\,{\scriptsize\up{11.4}} & \textbf{90.2}\,{\scriptsize\up{12.6}} & \textbf{55.4}\,{\scriptsize\up{9.6}} & \textbf{19.9}\,{\scriptsize\up{6.7}} & \textbf{60.9}\,{\scriptsize\up{10.0}} \\
\bottomrule
\end{tabular}%
}
% \vspace{-2mm}
\end{table}

We conduct ablation studies on the main components of \ours: co-evolved difficulty estimation with self-normalized importance sampling (\textit{SNIS}), difficulty-guided dynamic data selection (\textit{dynamic data selection}), concise reasoning for easy prompts (\textit{easy penalty}), and exploration incentive for hard prompts (\textit{hard bonus}). We evaluate these components with Qwen2.5-Math-1.5B and Qwen2.5-Math-7B on MATH-500, GSM8K, AIME-AMC, and Minerva. As shown in \autoref{tab:main_results} and \autoref{tab:ablation_1_5b}, we observe that \ding{182} each component improves final performance individually, and \ding{183} \ours achieves the best overall performance when all components are combined, indicating that these designs provide complementary gains rather than being redundant.

% \paragraph{Additional Invesitgations and Hyperparameter tuning.} \autoref{tab:el_hl_sweep} show the hyperparameter tuning results with $\lambda_{easy}$ and $\lambda_{hard}$, \autoref{tab:clip_sweep} show the tuning results of the clip hyperparameters $\epsilon_q^{-}$ and $\epsilon_q^{+}$. The results indicate that: \ding{182} A relatively small and symmetric coefficient of $\lambda_{easy}$ and $\lambda_{hard}$ of $1\time 10^{-4}$ achieves the best trade-off between effeciency and effectiveness; \ding{183} Relaxing only the upper bound $1+\epsilon_q^{+}$ to $1.6$ while keep the lower bound $1-\epsilon_q^{-}$ tight yields the best overall accuracy.

\paragraph{Additional investigations and hyperparameter tuning.}
\autoref{tab:el_hl_sweep} reports the sweep over the reward-shaping coefficients $\lambda_{\mathrm{easy}}$ and $\lambda_{\mathrm{hard}}$, while \autoref{tab:clip_sweep} reports the sweep over the clipping parameters $\epsilon_q^{-}$ and $\epsilon_q^{+}$. The results show that \ding{182} small and balanced coefficients, $\lambda_{\mathrm{easy}}=\lambda_{\mathrm{hard}}=1\times 10^{-4}$, provide the best trade-off between efficiency and effectiveness; and \ding{183} relaxing only the upper bound $1+\epsilon_q^{+}$ to $1.6$ while keeping the lower bound $1-\epsilon_q^{-}$ tight yields the best overall accuracy.

%% file: sec/6_conclusion.tex
\section{Conclusion}
\label{sec:conclusion}
We presented \ours, a difficulty-adaptive RL framework for efficient LLM reasoning. \ours addresses the limits of filtration-only methods by combining policy-aligned difficulty estimation, dynamic data selection, and difficulty-specific policy optimization. Across multiple models and benchmarks, \ours improves training efficiency, final accuracy, and inference-token efficiency, showing that effective RL should adapt both data selection and training signals to prompt difficulty.

%% file: sec/appendix.tex
\appendix

%% ─────────────────────────────────────────────────────────────────────
%% Appendix Table of Contents
%% ─────────────────────────────────────────────────────────────────────
{
\hypersetup{linkcolor=blue!70!black}

% Custom commands for appendix TOC entries with dotted leaders
\newcommand{\apptocline}[3]{%
  \noindent\hyperref[#3]{\textbf{#1}\hspace{0.6em}#2\,\dotfill\,\pageref{#3}}\par\vspace{1pt}%
}
\newcommand{\apptocsubline}[3]{%
  \noindent\hspace{1.5em}\hyperref[#3]{#1\hspace{0.5em}#2\,\dotfill\,\pageref{#3}}\par\vspace{1pt}%
}

\section*{Appendix Content}
\vspace{0.5em}

\apptocline{A}{Related Works}{sec:related}
\apptocline{B}{Proof of Proposition~\ref{prop:snis_finite_sample}}{sec:app_snis_finite_sample_proof}
\apptocline{C}{Full Algorithm}{sec:algorithm}
\apptocline{D}{Prompt Templates}{sec:prompt_templates}
\apptocline{E}{Difficulty Estimation Baselines}{sec:difficulty_estimation_baselines}
\apptocline{F}{Experimental Details}{sec:experimental_details}
\apptocline{G}{Extended Preliminary Investigation}{sec:comprehensive_preliminary_results}
\apptocline{H}{Ablation Studies and Additional Investigations}{sec:additional_invesitgations}
\apptocsubline{H.1}{Effective Sample Size of the SNIS Estimator}{sec:ess_threshold}
\apptocsubline{H.2}{Case Study: Stratified Generation Behavior}{sec:case_study}
\apptocsubline{H.3}{Entropy--Difficulty Relationship Across Models}{sec:entropy_difficulty}
\apptocsubline{H.4}{Ablation Results on Qwen2.5-Math-1.5B}{sec:ablation_1_5b}
\apptocsubline{H.5}{Additional Results on SmolLM3-3B}{sec:ablation_smollm}
\apptocline{I}{Generalization to Code Generation}{sec:coder_transfer}
\apptocline{J}{Training Efficiency Analysis}{sec:training_efficiency}
\apptocline{K}{Accuracy Under Constrained Rollout Budgets}{sec:budget_accuracy_appendix}
\apptocline{L}{Impact Statement}{sec:impact}
}

\clearpage

\section{Related Works}
\label{sec:related}
\paragraph{Reinforcement Learning for LLM Reasoning.}
With the widespread application of LLMs across various tasks \cite{zhou2025led, zhou2025m}, 
reinforcement learning has become a central approach for improving the reasoning capabilities of large language models. PPO~\cite{schulman2017proximal} and GRPO~\cite{shao2024deepseekmath} are among the most widely adopted algorithms, with GRPO eliminating the need for a separate value network by using group-relative advantages. These algorithms have driven significant progress on mathematical reasoning~\cite{shao2024deepseekmath,guo2025deepseek,cai2025role,jin2026reasoning,yuan2025mitigating,yuan2026behavior}, code generation~\cite{jain2025livecodebench_iclr,jin2025your}, and scientific question answering~\cite{rein2023gpqa,liu2026explainable}. However, RL training for LLMs remains costly and sample inefficient~\citep{yue2025does,tang2026highdataeff,liuleveraging}, since many generated rollouts produce near-zero advantages and contribute little to policy improvement. This inefficiency has motivated research into more effective data selection and compute allocation strategies.

\paragraph{Online Data Selection.}

A growing body of work shows that uniform prompt sampling wastes substantial compute on prompts that are either trivially easy or intractably hard for the current policy, both of which yield collapsed advantage signals~\cite{bae2026onlinedifficulty,li2025knapsackrl,zhang2025speedrl}. The prevailing strategy is to prioritize prompts of moderate difficulty, where the training signal is maximized~\cite{yu2025dapo,sun2025dots,zeng2026cures,bae2026onlinedifficulty,zhang2025speedrl}. Several methods further organize difficulty progression across training stages~\cite{chen2025selfevolving,parashar2026curriculum,shi2025efficientrl}, with extensions to code, multimodal, and writing tasks~\cite{deng2025boostingvlm,lei2025writingrl,wen2025lightr1}. Complementary work demonstrates that sample informativeness matters more than dataset scale~\cite{li2025limr,ye2025limo}, while theoretical analyses justify moderate-difficulty selection from the perspectives of advantage variance and gradient signal-to-noise ratio~\cite{bae2026onlinedifficulty,zeng2026cures,li2025knapsackrl}. Among these, DOTS~\cite{sun2025dots} is a particularly relevant baseline that combines an embedding-based difficulty predictor with rollout replay. However, prior studies also show that training exclusively on either easy or hard data does not consistently improve performance across the full range of task difficulties~\cite{parashar2026curriculum,kordi2026revisiting,yang2024can}, suggesting that focusing solely on moderate-difficulty prompts may be similarly insufficient. \ours goes beyond moderate-difficulty filtering by introducing co-evolved difficulty estimation that maintains accuracy under policy drift, and difficulty-adaptive training strategies that apply tailored objectives across easy, medium, and hard prompts to jointly improve inference efficiency and final effectiveness.

\paragraph{Difficulty Estimation.}

Accurate difficulty estimation is central to effective data selection, yet existing methods face a fundamental trade-off between cost and accuracy. Embedding-based predictors~\cite{sun2025dots} estimate difficulty via similarity to a fixed prompt space but cannot adapt as the policy evolves. Lightweight proxies based on entropy, perplexity, or uncertainty~\cite{pang2026edco,tang2026highdataeff,zhao2025uforl,shi2026intrinsic} are inexpensive but provide only coarse approximations. Bayesian posterior and bandit-style approaches~\cite{qu2025promptdifficulty,hu2025vade,shen2026bots} offer principled online estimation but typically require dedicated evaluation rollouts. Recent work also explores internal model representations or LLM-as-judge methods as difficulty signals~\cite{lee2025probing,zhu2025llmalreadyknows,tabib2025toward}. Despite these advances, many estimators rely on proxy signals or stale rollout statistics that become increasingly inaccurate as the policy changes during RL, as demonstrated in \autoref{sec:preliminary}.

\section{Proof of Proposition~\ref{prop:snis_finite_sample}}
\label{sec:app_snis_finite_sample_proof}

\begin{proof}
All probabilities and expectations below are conditional on the fixed estimation epoch, current policy, behavior policies, and clipping rule described before the proposition; we omit this conditioning to keep notation light. Write \(\mathbb{E}_k[\cdot]\triangleq\mathbb{E}_{O\sim\mu_k(\cdot\mid q)}[\cdot]\). Fix the prompt \(q\) throughout and write
\[
Y_k=Y(O_k),\qquad \omega_k=W_k(O_k),\qquad d_\theta=d_\theta(q).
\]
By absolute continuity, the likelihood ratio \(\rho_k(O)=\frac{d\pi_\theta(\cdot\mid q)}{d\mu_k(\cdot\mid q)}(O)\) satisfies the change-of-measure identities
\begin{equation}
\label{eq:app_change_of_measure_reward}
\mathbb{E}_{O\sim\mu_k(\cdot\mid q)}[\rho_k(O)Y(O)]
=
\mathbb{E}_{O\sim\pi_\theta(\cdot\mid q)}[Y(O)]
=
d_\theta,
\end{equation}
and
\begin{equation}
\label{eq:app_change_of_measure_weight}
\mathbb{E}_{O\sim\mu_k(\cdot\mid q)}[\rho_k(O)]=1.
\end{equation}

Define the empirical numerator and denominator
\begin{equation}
\label{eq:app_A_B}
A=\frac{1}{K}\sum_{k=1}^{K}\omega_kY_k,
\qquad
B=\frac{1}{K}\sum_{k=1}^{K}\omega_k,
\end{equation}
and their expectations
\begin{equation}
\label{eq:app_a_b}
a=\frac{1}{K}\sum_{k=1}^{K}
\mathbb{E}_{O\sim\mu_k(\cdot\mid q)}[W_k(O)Y(O)],
\qquad
b=\frac{1}{K}\sum_{k=1}^{K}
\mathbb{E}_{O\sim\mu_k(\cdot\mid q)}[W_k(O)].
\end{equation}
On the event \(B>0\), \(\hat d_q=A/B\), and the denominator assumption gives \(b\ge b_{\min}\).

Since \(0\le \omega_kY_k\le C_w\) and \(0\le \omega_k\le C_w\), the conditional independence assumption allows Hoeffding's inequality for independent bounded random variables, giving
\begin{equation}
\label{eq:app_hoeffding_A}
\mathbb{P}\left(|A-a|\ge \varepsilon\right)
\le
2\exp\left(-\frac{2K\varepsilon^2}{C_w^2}\right),
\end{equation}
and
\begin{equation}
\label{eq:app_hoeffding_B}
\mathbb{P}\left(|B-b|\ge \varepsilon\right)
\le
2\exp\left(-\frac{2K\varepsilon^2}{C_w^2}\right).
\end{equation}
Set \(\varepsilon=\varepsilon_\delta=C_w\sqrt{\log(4/\delta)/(2K)}\). A union bound implies that, with probability at least \(1-\delta\),
\begin{equation}
\label{eq:app_good_event}
|A-a|\le \varepsilon_\delta,
\qquad
|B-b|\le \varepsilon_\delta.
\end{equation}
On this event, \(\varepsilon_\delta<b_{\min}\) implies
\begin{equation}
\label{eq:app_denominator_lower}
B\ge b-\varepsilon_\delta\ge b_{\min}-\varepsilon_\delta>0.
\end{equation}
Moreover, because \(0\le Y\le 1\), we have \(0\le a\le b\). Therefore,
\begin{align}
\left|\frac{A}{B}-\frac{a}{b}\right|
&=
\left|
\frac{b(A-a)-a(B-b)}{bB}
\right| \nonumber\\
&\le
\frac{b|A-a|+a|B-b|}{bB} \nonumber\\
&\le
\frac{2\varepsilon_\delta}{B}
\le
\frac{2\varepsilon_\delta}{b_{\min}-\varepsilon_\delta}.
\label{eq:app_ratio_concentration}
\end{align}

It remains to compare \(a/b\) with the current-policy difficulty \(d_\theta\). Using \autoref{eq:app_change_of_measure_reward} and \autoref{eq:app_change_of_measure_weight},
\begin{equation}
\label{eq:app_a_decomp}
a
=
d_\theta
-
\frac{1}{K}\sum_{k=1}^{K}
\mathbb{E}_{O\sim\mu_k(\cdot\mid q)}
\left[(\rho_k(O)-W_k(O))Y(O)\right],
\end{equation}
and
\begin{equation}
\label{eq:app_b_decomp}
b
=
1
-
\frac{1}{K}\sum_{k=1}^{K}
\mathbb{E}_{O\sim\mu_k(\cdot\mid q)}
\left[\rho_k(O)-W_k(O)\right].
\end{equation}
Let
\[
\begin{aligned}
S&=
\frac{1}{K}\sum_{k=1}^{K}
\mathbb{E}_{O\sim\mu_k(\cdot\mid q)}
\left[\rho_k(O)-W_k(O)\right],\\
T&=
\frac{1}{K}\sum_{k=1}^{K}
\mathbb{E}_{O\sim\mu_k(\cdot\mid q)}
\left[(\rho_k(O)-W_k(O))Y(O)\right].
\end{aligned}
\]
Then \(b=1-S\), \(a=d_\theta-T\), and \(b\ge b_{\min}\). Since \(Y(O)\in[0,1]\), we also have \(d_\theta\in[0,1]\), and hence \(|d_\theta-Y(O)|\le1\). Thus
\begin{align}
\left|\frac{a}{b}-d_\theta\right|
&=
\left|
\frac{d_\theta-T}{1-S}-d_\theta
\right| \nonumber\\
&=
\frac{|d_\theta S-T|}{1-S}
\le
\frac{1}{b_{\min}K}
\sum_{k=1}^{K}
\mathbb{E}_{O\sim\mu_k(\cdot\mid q)}
\left[
\left|\rho_k(O)-W_k(O)\right|
\left|d_\theta-Y(O)\right|
\right] \nonumber\\
&\le
\frac{1}{b_{\min}K}
\sum_{k=1}^{K}
\mathbb{E}_{O\sim\mu_k(\cdot\mid q)}
\left[
\left|\rho_k(O)-W_k(O)\right|
\right].
\label{eq:app_clipped_target_gap}
\end{align}
Combining \eqref{eq:app_ratio_concentration} and \eqref{eq:app_clipped_target_gap} by the triangle inequality gives \eqref{eq:finite_sample_bound}. Finally, for the two-sided log clipping in \autoref{eq:is_weights},
\[
W_k(O)=\min\{\max\{\rho_k(O),e^{-c}\},e^c\},
\]
and therefore
\[
\left|\rho_k(O)-W_k(O)\right|
=
(\rho_k(O)-e^c)_+ +(e^{-c}-\rho_k(O))_+,
\]
which yields the stated log-clipping special case.
\end{proof}

\section{Full Algorithm}
\label{sec:algorithm}

The complete procedure of \ours is summarized in \autoref{alg:ours}. At the start of training, a small reference set is rolled out to bootstrap embedding-based difficulty estimates for unseen prompts (lines 2--4). In each epoch, \ours first updates the difficulty estimate $\hat{d}_q$ for every prompt: cold-start prompts use similarity-weighted interpolation over the reference set, while prompts with buffered rollouts are corrected for policy drift via clipped self-normalized importance sampling (lines 7--14). These estimates then define a Beta-shaped sampling distribution that concentrates each training batch on medium-difficulty prompts while retaining nonzero probability across the full difficulty spectrum (lines 15--16). Within the sampled batch, each prompt is routed to one of three optimization regimes according to its estimated difficulty: easy prompts receive fewer rollouts and a length penalty that rewards concise correct reasoning; medium prompts undergo standard GRPO updates; and hard prompts are allocated additional rollouts—partly hint-augmented—together with a length bonus that encourages deeper exploration (lines 18--30). Finally, the policy is updated on a mixture of fresh and replayed trajectories with difficulty-conditioned clipping, and new rollouts are pushed into the replay buffer for future reuse (lines 31--33).

\begin{algorithm}[!ht]
\caption{\ours: Difficulty-Adaptive Reinforcement Learning with Co-Evolved Difficulty Estimation}
\label{alg:ours}
\begin{algorithmic}[1]
\REQUIRE Training set $\mathcal{D}$, reference set $\mathcal{D}_{\mathrm{ref}}$, initial policy $\pi_\theta$, replay buffer $\mathcal{B}\leftarrow\emptyset$, Beta concentration $\kappa$, thresholds $d_{\mathrm{easy}}, d_{\mathrm{hard}}$, rollout counts $G, G_{\mathrm{easy}}, G_{\mathrm{hard}}$, replay mix ratio $\sigma$, clipping bound $c$
\STATE \textcolor{gray}{\texttt{/* Cold-start: initialize reference difficulties */}}
\FOR{each reference prompt $p_i \in \mathcal{D}_{\mathrm{ref}}$}
    \STATE Generate $G$ rollouts from $\pi_\theta(\cdot \mid p_i)$; compute $d_{p_i} \leftarrow 1 - \bar{r}_{p_i}$ \COMMENT{Eq.~\eqref{eq:success_rate_difficulty}}
\ENDFOR
\FOR{each training epoch}
    \STATE \textcolor{gray}{\texttt{/* Phase 1: Co-Evolved Difficulty Estimation (\S\ref{sec:is_estimation}) */}}
    \FOR{each prompt $q \in \mathcal{D}$}
        \IF{$\mathcal{B}_q = \emptyset$}
            \STATE Compute embedding similarities $a_i$; set $\hat{d}_q \leftarrow \sum_i a_i\, d_{p_i}$ \COMMENT{Eq.~\eqref{eq:embedding_weight_sum}}
        \ELSE
            \FOR{each rollout $(o_k, r_k, \log\pi_{\theta_{\mathrm{beh},k}})$ in $\mathcal{B}_q$}
                \STATE $w_k \leftarrow \exp\!\bigl(\mathrm{clip}(\log\pi_\theta(o_k\mid q) - \log\pi_{\theta_{\mathrm{beh},k}}(o_k\mid q),\, {-c},\, c)\bigr)$ \COMMENT{Eq.~\eqref{eq:is_weights}}
            \ENDFOR
            \STATE $\hat{d}_q \leftarrow \sum_k w_k(1-r_k)\big/\sum_k w_k$ 
        \ENDIF
    \ENDFOR
    \STATE \textcolor{gray}{\texttt{/* Phase 2: Difficulty-Guided Dynamic Data Selection (\S\ref{sec:sample_selection}) */}}
    \STATE Set $p_{\ours}(q) \propto \mathrm{Beta}(\hat{d}_q;\, 1+\kappa/2,\, 1+\kappa/2)$ for all $q$ \COMMENT{Eq.~\eqref{eq:beta_sampling}}
    \STATE Sample batch $\mathcal{Q}$ from $\mathcal{D}$ without replacement according to $p_{\ours}$
    \STATE \textcolor{gray}{\texttt{/* Phase 3: Difficulty-Adaptive Policy Optimization (\S\ref{sec:stratified_rl}) */}}
    \FOR{each prompt $q \in \mathcal{Q}$}
        \STATE Assign tier: \textbf{easy} if $\hat{d}_q < d_{\mathrm{easy}}$;\ \textbf{hard} if $\hat{d}_q > d_{\mathrm{hard}}$;\ \textbf{medium} otherwise
        \STATE \textcolor{gray}{\texttt{/* }\iconeasybig\ \texttt{Easy: Concise reasoning */}}
        \IF{$\hat{d}_q < d_{\mathrm{easy}}$}
        \STATE Set $G_q \leftarrow G_{\mathrm{easy}}$, impose shorter response budget
        \STATE Shape reward: $\tilde{r}_i \leftarrow r_i - \mathbf{1}_{r_i=1}\,\lambda_{\mathrm{easy}}\, w_q^{\mathrm{easy}}\, |o_i|/T_{\max}$ \COMMENT{Eq.~\eqref{eq:easy_shaped_reward}}
        \STATE Set $\epsilon_q^{-} \leftarrow \epsilon$,\ $\epsilon_q^{+} > \epsilon$ \COMMENT{relaxed upper clip}
        \ENDIF
        \STATE \textcolor{gray}{\texttt{/* }\iconmediumbig\ \texttt{Medium: Standard GRPO */}}
        \IF{$d_{\mathrm{easy}} \leq \hat{d}_q \leq d_{\mathrm{hard}}$}
        \STATE Set $G_q \leftarrow G$,\ $\tilde{r}_i \leftarrow r_i$,\ $\epsilon_q^{-} \leftarrow \epsilon_q^{+} \leftarrow \epsilon$
        \ENDIF
        \STATE \textcolor{gray}{\texttt{/* }\iconhardbig\ \texttt{Hard: Exploration incentive */}}
        \IF{$\hat{d}_q > d_{\mathrm{hard}}$}
        \STATE Set $G_q \leftarrow G_{\mathrm{hard}}$; generate $G$ standard $+$ $(G_q - G)$ hint-augmented rollouts
        \STATE Shape reward: $\tilde{r}_i \leftarrow r_i + \mathbf{1}_{r_i=0}\,\lambda_{\mathrm{hard}}\, w_q^{\mathrm{hard}}\, |o_i|/T_{\max}$ \COMMENT{Eq.~\eqref{eq:hard_shaped_reward}}
        \ENDIF
        \STATE Compute group-relative advantages $A_i$ from $\{\tilde{r}_i\}_{i=1}^{G_q}$
    \ENDFOR
    \STATE Construct batch: mix fraction $\sigma$ fresh on-policy rollouts with $(1-\sigma)$ replay trajectories
    \STATE Update $\theta$ by maximizing $\mathcal{J}_{\ours}(\theta)$ with difficulty-conditioned clipping \COMMENT{Eq.~\eqref{eq:ours_grpo}}
    \STATE Push new rollouts into $\mathcal{B}$ (FIFO eviction)
\ENDFOR
\RETURN Optimized policy $\pi_\theta$
\end{algorithmic}
\end{algorithm}

\section{Prompt Templates}
\label{sec:prompt_templates}

All difficulty levels share the same system prompt. For hard prompts, additional hints are appended to the user content.

\paragraph{System prompt (all difficulty levels).}
\begin{quote}
\texttt{Let's think step by step and output the final answer within \textbackslash boxed\{\}.}
\end{quote}
For GSM8K-style tasks, the system prompt is:
\begin{quote}
\texttt{Let's think step by step and output the final answer after "\#\#\#\#".}
\end{quote}

\paragraph{Think-longer rollout.}
The following hint is appended to the user content for hard prompts ($\hat{d}_q > d_{\mathrm{hard}}$):
\begin{quote}
\texttt{This is a challenging problem. Please think longer and more carefully. Break it down step by step before giving your final answer.}
\end{quote}

\paragraph{Hint-augmented Prompt.}
For hard prompts with a successful historical trajectory available in the replay buffer, the user content is augmented with a reasoning skeleton extracted from the successful response:
\begin{quote}
\texttt{[Original question]}\\[4pt]
\texttt{Helpful hint from one previous successful attempt on this same problem:}\\
\texttt{- Guided reasoning skeleton:}\\
\texttt{\quad 1. [Top-scoring fragment 1]}\\
\texttt{\quad 2. [Top-scoring fragment 2]}\\
\texttt{\quad 3. [Top-scoring fragment 3]}\\[4pt]
\texttt{Use the hint as guidance, but still reason carefully and give your own complete solution within \textbackslash boxed\{\}.}
\end{quote}

\section{Difficulty Estimation Baselines}
\label{sec:difficulty_estimation_baselines}

All methods below share the same prompt selection rule: given predicted difficulty labels $\hat{d}_q$ for every prompt $q \in \mathcal{D}$, sampling weights are computed as $p(q) \propto \mathrm{Beta}(\hat{d}_q;\,\alpha,\alpha)$ with $\alpha = 1 + \kappa/2$ (Eq.~\eqref{eq:beta_sampling}), and $B$ prompts are drawn via multinomial sampling. The methods differ only in how $\hat{d}_q$ is obtained.

%%% --- 1. Random Selection ---
\paragraph{Random Selection.}
No difficulty estimation is performed. All prompts are assigned a uniform predicted label $\hat{d}_q = 0.5$, and $B$ prompts are drawn uniformly at random from $\mathcal{D}$ without replacement.

\begin{algorithm}[H]
\caption{\textit{Random Selection}}
\label{alg:random}
\begin{algorithmic}[1]
\REQUIRE Training prompts $\mathcal{D}$, selection budget $B$
\FOR{each training step}
  \STATE $\hat{d}_q \gets 0.5 \quad \forall\, q \in \mathcal{D}$
  \STATE Draw $B$ prompts uniformly at random: $\mathcal{S} \gets \textsc{RandomSample}(\mathcal{D}, B)$
  \STATE Roll out $G$ responses per prompt in $\mathcal{S}$; proceed to GRPO update
\ENDFOR
\end{algorithmic}
\end{algorithm}

%%% --- 2. Embedding-Based Prediction ---
\paragraph{Embedding-Based Prediction.}
This method follows the teacher-model framework of \citet{sun2025dots}. At the beginning of each training step, a reference set $\mathcal{R} \subset \mathcal{D}$ of $n_{\mathrm{ref}}$ prompts is sampled uniformly and rolled out with the current policy $\pi_\theta$ to obtain ground-truth success rates. A lightweight teacher model $f_\phi$---a 3-layer residual MLP head attached to the frozen last-hidden-state embeddings of Qwen2.5-Math-1.5B-Instruct---computes the cosine similarity between each remaining prompt and the reference set in embedding space, and predicts the success rate via group-logit-temperature scaling. The predicted difficulty is $\hat{d}_q = 1 - \hat{s}_q^{\mathrm{teacher}}$.

\begin{algorithm}[H]
\caption{\textit{Embedding-Based Prediction}~\citep{sun2025dots}}
\label{alg:embedding}
\begin{algorithmic}[1]
\REQUIRE Training prompts $\mathcal{D}$, teacher model $f_\phi$, reference size $n_{\mathrm{ref}}$, selection budget $B$
\FOR{each training step}
  \STATE Sample reference set $\mathcal{R} \subset \mathcal{D}$, $|\mathcal{R}| = n_{\mathrm{ref}}$
  \STATE Roll out $G$ responses per prompt in $\mathcal{R}$; compute ground-truth $s_q$ for $q \in \mathcal{R}$
  \STATE Extract embeddings $\mathbf{h}_q$ for all $q \in \mathcal{D}$ from the frozen encoder
  \FOR{each $q \in \mathcal{D} \setminus \mathcal{R}$}
    \STATE $\hat{s}_q^{\mathrm{teacher}} \gets f_\phi(\mathbf{h}_q,\, \{\mathbf{h}_{q'}, s_{q'}\}_{q' \in \mathcal{R}})$
    \STATE $\hat{d}_q \gets 1 - \hat{s}_q^{\mathrm{teacher}}$
  \ENDFOR
  \STATE For $q \in \mathcal{R}$: $\hat{d}_q \gets 1 - s_q$
  \STATE Select $B$ prompts via Beta sampling on $\{\hat{d}_q\}$
\ENDFOR
\end{algorithmic}
\end{algorithm}

%%% --- 3. Entropy-Based Estimation ---
\paragraph{Entropy-Based Estimation.}
Following \citet{pang2026edco}, this method uses the token-level entropy of the current policy as a proxy for difficulty. For each prompt $q$, a single short prefix of $L_{\mathrm{prefix}}$ tokens is generated from $\pi_\theta$ at temperature $T_{\mathrm{ent}}$. The token-level entropy $H_t = -\sum_v \pi_\theta(v \mid q, o_{<t}) \log \pi_\theta(v \mid q, o_{<t})$ is computed at each generated position, and the prompt-level score is the mean $\bar{H}_q = \frac{1}{L_{\mathrm{prefix}}} \sum_{t=1}^{L_{\mathrm{prefix}}} H_t$. Prompts are then selected via softmax sampling with temperature $\tau_{\mathrm{ent}}$: $p(q) \propto \exp(\bar{H}_q / \tau_{\mathrm{ent}})$, which favors high-entropy (uncertain) prompts.

\begin{algorithm}[H]
\caption{\textit{Entropy-Based Estimation}~\citep{pang2026edco}}
\label{alg:entropy}
\begin{algorithmic}[1]
\REQUIRE Training prompts $\mathcal{D}$, prefix length $L_{\mathrm{prefix}}$, temperature $T_{\mathrm{ent}}$, softmax temperature $\tau_{\mathrm{ent}}$, budget $B$
\FOR{each training step}
  \FOR{each $q \in \mathcal{D}$}
    \STATE Generate prefix $o_{1:L_{\mathrm{prefix}}} \sim \pi_\theta(\cdot \mid q)$ with temperature $T_{\mathrm{ent}}$
    \STATE Compute $\bar{H}_q \gets \frac{1}{L_{\mathrm{prefix}}} \sum_{t=1}^{L_{\mathrm{prefix}}} H_t$
  \ENDFOR
  \STATE $p(q) \propto \exp\!\bigl(\bar{H}_q \,/\, \tau_{\mathrm{ent}}\bigr)$
  \STATE Draw $B$ prompts via multinomial sampling from $p$
\ENDFOR
\end{algorithmic}
\end{algorithm}

%%% --- 4. LLM-Judge ---
\paragraph{LLM-Judge.}
Inspired by \citet{tabib2025toward}, this method uses the policy model $\pi_\theta$ itself to judge prompt difficulty. For each prompt $q$, a difficulty-estimation instruction is prepended as a system message, asking the model to output a success probability $\hat{s}_q \in [0, 1]$. The model generates a short response (up to 32 tokens) at temperature $T_{\mathrm{judge}} = 0.3$, and the first valid floating-point number in the output is parsed as $\hat{s}_q$. The predicted difficulty is $\hat{d}_q = 1 - \hat{s}_q$. Prompts for which parsing fails are assigned $\hat{d}_q = 0.5$.

\begin{algorithm}[H]
\caption{\textit{LLM-Judge}~\citep{tabib2025toward}}
\label{alg:llm_judge}
\begin{algorithmic}[1]
\REQUIRE Training prompts $\mathcal{D}$, policy $\pi_\theta$, judge temperature $T_{\mathrm{judge}}$, budget $B$
\FOR{each training step}
  \FOR{each $q \in \mathcal{D}$}
    \STATE Construct judge prompt: system instruction $+$ $q$
    \STATE Generate response $r \sim \pi_\theta(\cdot \mid \text{judge prompt})$ with $T = T_{\mathrm{judge}}$, max 32 tokens
    \STATE Parse $\hat{s}_q$ from $r$; set $\hat{s}_q \gets 0.5$ on parse failure
    \STATE $\hat{d}_q \gets 1 - \hat{s}_q$
  \ENDFOR
  \STATE Select $B$ prompts via Beta sampling on $\{\hat{d}_q\}$
\ENDFOR
\end{algorithmic}
\end{algorithm}

%%% --- 5. Bayesian Posterior Estimation ---
\paragraph{Bayesian Posterior Estimation.}
Following \citet{qu2025promptdifficulty}, each prompt $q$ is modeled as a Bernoulli arm with latent success rate $\gamma_q \sim \mathrm{Beta}(\alpha_q, \beta_q)$. The posterior is initialized with a uniform prior $(\alpha_0, \beta_0) = (1, 1)$. After each rollout batch where prompt $q$ receives $k$ rollouts with $s$ successes, the posterior is updated with temporal discounting factor $\lambda$:
\begin{equation}
\alpha_q \gets \lambda\,\alpha_q + (1-\lambda)\,\alpha_0 + s, \qquad
\beta_q \gets \lambda\,\beta_q + (1-\lambda)\,\beta_0 + (k - s).
\label{eq:bayesian_update}
\end{equation}
Difficulty is estimated via Thompson Sampling: $\tilde{\gamma}_q \sim \mathrm{Beta}(\alpha_q, \beta_q)$, and $\hat{d}_q = 1 - \tilde{\gamma}_q$. For prompts not yet covered by any rollout, a default label $\hat{d}_q = 0.5$ is assigned, and an exploration ratio $\rho_{\mathrm{explore}}$ controls the probability mass allocated to uncovered prompts.

\begin{algorithm}[H]
\caption{\textit{Bayesian Posterior Estimation}~\citep{qu2025promptdifficulty}}
\label{alg:bayesian}
\begin{algorithmic}[1]
\REQUIRE Training prompts $\mathcal{D}$, prior $(\alpha_0, \beta_0)$, decay $\lambda$, explore ratio $\rho_{\mathrm{explore}}$, budget $B$
\STATE Initialize $\alpha_q \gets \alpha_0,\; \beta_q \gets \beta_0 \quad \forall\, q \in \mathcal{D}$
\FOR{each training step}
  \FOR{each $q \in \mathcal{D}$ with rollout history}
    \STATE Draw $\tilde{\gamma}_q \sim \mathrm{Beta}(\alpha_q, \beta_q)$
    \STATE $\hat{d}_q \gets 1 - \tilde{\gamma}_q$
  \ENDFOR
  \STATE For uncovered prompts: $\hat{d}_q \gets 0.5$; allocate $\rho_{\mathrm{explore}}$ of sampling mass
  \STATE Select $B$ prompts via Beta sampling on $\{\hat{d}_q\}$
  \STATE Roll out $G$ responses per selected prompt; observe $(s_q, k_q)$
  \STATE Update posteriors via Eq.~\eqref{eq:bayesian_update}
\ENDFOR
\end{algorithmic}
\end{algorithm}

%%% --- 6. Previous FR Estimation ---
\paragraph{Previous FR Estimation.}
This method uses the empirical failure rate from previous rollouts stored in the replay buffer $\mathcal{B}$ as the difficulty estimate, without importance-weight correction. For each prompt $q$ with $K$ stored rollouts $\{r_k\}_{k=1}^{K}$ in $\mathcal{B}$, the predicted failure rate is $\hat{f}_q^{\mathrm{raw}} = 1 - \frac{1}{K}\sum_{k=1}^{K} r_k$, and the difficulty is $\hat{d}_q = \hat{f}_q^{\mathrm{raw}}$. For prompts without buffer entries, a default label $\hat{d}_q = 0.5$ is used. This is equivalent to the SNIS estimator (\autoref{sec:is_estimation}) with all importance weights set to unity, i.e., ignoring the distributional shift between the behavior policies and the current policy.

\begin{algorithm}[H]
\caption{\textit{Previous FR Estimation}}
\label{alg:prev_sr}
\begin{algorithmic}[1]
\REQUIRE Training prompts $\mathcal{D}$, replay buffer $\mathcal{B}$, budget $B$
\FOR{each training step}
  \FOR{each $q \in \mathcal{D}$ with entries in $\mathcal{B}$}
    \STATE Retrieve stored rollouts $\{r_k\}_{k=1}^{K}$ from $\mathcal{B}(q)$
    \STATE $\hat{f}_q^{\mathrm{raw}} \gets 1 - \frac{1}{K}\sum_{k=1}^{K} r_k$
    \STATE $\hat{d}_q \gets \hat{f}_q^{\mathrm{raw}}$
  \ENDFOR
  \STATE For prompts without buffer entries: $\hat{d}_q \gets 0.5$
  \STATE Select $B$ prompts via Beta sampling on $\{\hat{d}_q\}$
\ENDFOR
\end{algorithmic}
\end{algorithm}

%%% --- 7. Current FR Estimation ---
\paragraph{Current FR Estimation.}
This method serves as the ground-truth oracle for difficulty estimation. At each training step, every prompt $q \in \mathcal{D}$ is rolled out $G$ times with the current policy $\pi_\theta$ at temperature $T$, yielding the observed failure rate $f_q = 1 - \frac{1}{G}\sum_{i=1}^{G} r_i$. Only prompts whose observed failure rate falls within a target interval $[f_{\mathrm{lo}}, f_{\mathrm{hi}}]$ centered at 0.5 are retained; rollouts continue until $B$ qualifying prompts are collected. The difficulty is $\hat{d}_q = f_q$.

\begin{algorithm}[H]
\caption{\textit{Current FR Estimation} (Ground Truth)}
\label{alg:current_sr}
\begin{algorithmic}[1]
\REQUIRE Training prompts $\mathcal{D}$, policy $\pi_\theta$, group size $G$, target interval $[f_{\mathrm{lo}}, f_{\mathrm{hi}}]$, budget $B$
\FOR{each training step}
  \STATE $\mathcal{S} \gets \emptyset$
  \WHILE{$|\mathcal{S}| < B$}
    \STATE Sample a batch of prompts from $\mathcal{D}$
    \FOR{each sampled prompt $q$}
      \STATE Generate $G$ rollouts $\{o_i\} \sim \pi_\theta(\cdot \mid q)$; compute $f_q = 1 - \frac{1}{G}\sum_{i=1}^{G} r_i$
      \IF{$f_q \in [f_{\mathrm{lo}},\, f_{\mathrm{hi}}]$}
        \STATE $\mathcal{S} \gets \mathcal{S} \cup \{q\}$; $\hat{d}_q \gets f_q$
      \ENDIF
    \ENDFOR
  \ENDWHILE
  \STATE Proceed to GRPO update with $\mathcal{S}$
\ENDFOR
\end{algorithmic}
\end{algorithm}

\section{Experimental Details}
\label{sec:experimental_details}
\label{sec:preliminary_investigation_details}

% \subsection{Preliminary Investigation Details}

The preliminary investigation in \autoref{sec:preliminary} trains Qwen2.5-Math-1.5B~\citep{yang2024qwen2} on the same four-dataset union as the main experiments, following the data preparation of \citet{sun2025dots}: MATH Level~3--5~\citep{hendrycks2021measuring} ($8{,}888$ prompts merged from the train and test splits, with the MATH-500 evaluation set removed), DeepScaleR-40K~\citep{luo2025deepscaler} ($10{,}240$ prompts subsampled from 40K), Open-Reasoner-Zero-57K~\citep{hu2025openreasonerzero} ($8{,}192$ prompts subsampled from 57K), and DeepMath-103K~\citep{he2025deepmath} ($8{,}192$ prompts subsampled from 103K), for a total of $35{,}512$ prompts. Subsampling uses a fixed seed for reproducibility; each dataset is converted to a common prompt format with a shared system prompt and a rule-based verifiable reward. Training uses GRPO~\citep{shao2024deepseekmath} with the AdamW optimizer (learning rate $2 \times 10^{-6}$, no weight decay), group size $G = 8$, generation temperature $T = 0.6$, maximum response length $3{,}072$ tokens, and entropy coefficient $0$. Evaluation is performed on MATH-500~\citep{lightman2023let} and GSM8K~\citep{cobbe2021training}.

\paragraph{Shared parameters for teacher-based and filtration methods.}
\textit{Random Selection}, \textit{Embedding-Based Prediction}, \textit{Entropy-Based Estimation}, \textit{LLM-Judge}, and \textit{Current FR Estimation} use an effective batch size $B = 512$, on-policy fraction $\sigma = 0.5$ (i.e., 256 fresh rollouts and 256 replay samples per step), and a replay buffer of capacity $C = 512$. Prompt selection uses a Beta sampling distribution with concentration $\kappa = 100$ centered at $\alpha = 0.5$ (Eq.~\eqref{eq:beta_sampling}).

\paragraph{Shared parameters for buffer-based methods.}
\textit{Previous FR Estimation} and \textit{Bayesian Posterior Estimation} use $B = 512$, $\sigma = 1.0$ (fully on-policy), and a larger replay buffer $C = 4096$ to accumulate sufficient rollout history for posterior estimation.

\paragraph{Method-specific parameters.}
\begin{itemize}[leftmargin=*,nosep]
  \item \textbf{Embedding-Based Prediction}: reference set size $n_{\mathrm{ref}} = 256$; teacher model is a 3-layer residual MLP head on top of frozen Qwen2.5-Math-1.5B-Instruct embeddings, with group-logit-temperature scaling; teacher batch size 32.
  \item \textbf{Entropy-Based Estimation}: prefix length $L_{\mathrm{prefix}} = 64$ tokens, generation temperature $T_{\mathrm{ent}} = 0.6$, softmax temperature $\tau_{\mathrm{ent}} = 1.0$. The entropy-based selection replaces the Beta sampling rule with softmax sampling over raw entropy scores.
  \item \textbf{LLM-Judge}: the policy model itself generates difficulty predictions at temperature $T_{\mathrm{judge}} = 0.3$ with a maximum of 32 output tokens per prompt.
  \item \textbf{Bayesian Posterior Estimation}: uniform prior $(\alpha_0, \beta_0) = (1.0, 1.0)$, temporal decay $\lambda = 0.5$, target failure rate $\gamma^* = 0.5$, exploration ratio $\rho_{\mathrm{explore}} = 0.3$.
  \item \textbf{Previous FR Estimation}: uses the same replay buffer as the SNIS estimator ($C = 4096$, $\sigma = 1.0$) but computes the raw average failure rate without importance-weight correction.
  \item \textbf{Current FR Estimation}: rolls out $G = 8$ responses per prompt with the current policy at temperature $T = 0.6$; retains only prompts with observed failure rate $f_q \in [0.4, 0.6]$ (i.e., $f_q = 0.5$ given $G = 8$), continuing until $B$ qualifying prompts are collected.
\end{itemize}

\paragraph{SNIS difficulty estimation defaults (main experiments).}
Runs that use SNIS difficulty estimation (\autoref{sec:is_estimation}) in \autoref{sec:experiments} use the following defaults unless explicitly overridden in an ablation: replay buffer size $4{,}096$, Beta sampling concentration $\kappa = 100$, SNIS clipping range $c = 4.0$, ESS threshold $\tau_{\mathrm{ESS}} = 3$, EasyLength penalty coefficient $\lambda_{\mathrm{easy}} = 10^{-4}$, and hard length bonus coefficient $\lambda_{\mathrm{hard}} = 10^{-4}$ when reward shaping is enabled.

% \paragraph{Learning rate and scheduler selection.}
% The learning rate and scheduler used in \autoref{sec:experiments} were selected via a sweep over $\{2{\times}10^{-4}, 10^{-3}, 2{\times}10^{-3}\}$ with and without a cosine schedule; per-configuration training curves are reported in \appref{sec:loss_ablation}.

% \paragraph{AIME-AMC and AIME24/25 composition.}
% \textit{AIME-AMC}~\citep{aimo2024aime,aimo2024amc} is the AI-MO validation benchmark, consisting of $30$ historical AIME problems and $83$ AMC problems ($113$ problems in total). \textit{AIME24/25}~\citep{maa2024aime} is a separate $60$-problem benchmark that combines the $30$ AIME~2024 problems and the $30$ AIME~2025 problems. The two benchmarks draw from disjoint competition years and do not overlap. AIME24/25 is reported separately in \autoref{tab:el_hl_sweep} and \autoref{tab:clip_sweep}.

\section{Extended Preliminary Investigation}
\label{sec:comprehensive_preliminary_results}

\autoref{tab:difficulty_accuracy} reports the MSE and MAE of each difficulty estimation method on $200$ held-out prompts, measuring how close the selected prompts are to the target accuracy of $0.5$. The results show that \textit{Our Estimation} achieves the lowest error among all practical estimation methods, reducing MAE from $0.2675$ for the strongest existing baseline, \textit{Embedding-Based Prediction}, to $0.1894$. Its error is also close to the oracle \textit{Current FR Estimation}, indicating that our policy-aligned estimator better tracks the current policy difficulty during RL training.

\begin{table}[h]
\centering
\small
\caption{Estimation error (MSE and MAE) of the predicted accuracy distribution relative to the target accuracy of 0.5 on 200 held-out prompts. Lower is better.}
\label{tab:difficulty_accuracy}
\begin{tabular}{lcc}
\toprule
Method & MSE$\downarrow$ & MAE$\downarrow$ \\
\midrule
\textit{Random Selection} & 0.1624 & 0.3719 \\
\textit{Entropy-Based Estimation} & 0.1288 & 0.3120 \\
\textit{Previous FR Estimation} & 0.1622 & 0.2988 \\
\textit{LLM-Judge} & 0.1568 & 0.3688 \\
\textit{Bayesian Posterior Estimation} & 0.0964 & 0.2839 \\
\textit{Embedding-Based Prediction} & 0.0943 & 0.2775 \\
\textit{Current FR Estimation} (Ground Truth) & 0.0473 & 0.1688 \\
\midrule
\textit{Our Estimation} & 0.0654 & 0.1894 \\
\bottomrule
\end{tabular}
\end{table}

% \begin{table}[h]
% \centering
% \caption{Downstream evaluation accuracy across six benchmarks for different difficulty estimation methods trained on Qwen2.5-Math-1.5B.}
% \label{tab:difficulty_downstream}
% \small
% \resizebox{\textwidth}{!}{%
% \begin{tabular}{lcccccc}
% \toprule
% Method & AIME24/25 & MATH-500 & GSM8K & AIME-AMC & Minerva & Olympiad \\
% \midrule
% \rowcolor{gray!8}
% % Qwen2.5-Math-1.5B & 6.7 & 34.2 & 47.6 & 31.3 & 4.4 & 15.4 \\
% \textit{GRPO} & 11.7 & 66.8 & 77.6 & 45.8 & 13.2 & 22.6 \\
% \textit{LLM-Judge} & 8.3 & 69.8 & 79.1 & 42.2 & 13.2 & 24.0 \\
% \textit{DOTS} & 6.7 & 69.4 & 79.4 & 45.8 & 15.4 & 23.7 \\
% \textit{EDCO} & 5.0 & 69.6 & 80.6 & 47.0 & 13.6 & 24.6 \\
% \textit{MoPPS} & 0.0 & 69.7 & 80.9 & 44.4 & 14.8 & 20.0 \\
% % \textit{Current FR Est.} & 5.0 & 68.8 & 76.2 & 47.0 & 17.3 & 25.5 \\
% \textit{Previous FR Est.} & 11.7 & 67.9 & 76.5 & 44.8 & 13.2 & 23.2 \\

% \bottomrule
% \end{tabular}%
% }
% \end{table}

\begin{table}[t]
\centering
\small
\caption{Average and peak GPU memory, together with average and peak CPU memory, under different difficulty estimation methods.}
\label{tab:prelim_memory_usage}
\begin{tabular}{lcccc}
\toprule
Method & Avg. GPU Mem. & Peak GPU Mem. & Avg. CPU Mem. & Peak CPU Mem. \\
\midrule
\textit{GRPO} & 26.62 & 50.46 & 6.50 & 10.26 \\
\textit{EDCO} & 26.61 & 50.50 & 7.80 & 14.50 \\
\textit{LLM-Judge} & 26.60 & 50.43 & 7.33 & 13.15 \\
\textit{DOTS} & 26.64 & 50.45 & 7.08 & 12.51 \\
\textit{Previous FR Estimation} & 23.75 & 47.94 & 9.32 & 18.85 \\
\textit{MoPPS} & 26.62 & 50.50 & 7.80 & 14.50 \\
% \hline
\textit{\ours} & 25.75 & 50.94 & 7.32 & 15.65 \\
\bottomrule
\end{tabular}
\end{table}

% Moved to main text (\autoref{sec:main_results}).

\section{Ablation Studies and Additional Investigations}
\label{sec:additional_invesitgations}

\subsection{Effective sample size of the SNIS estimator}
\label{sec:ess_threshold}
To decide whether the SNIS estimate is reliable enough to be used as the difficulty signal, we monitor the effective sample size computed from the clipped importance weights $\{w_k\}_{k=1}^{K}$ (Eq.~\eqref{eq:is_weights}) associated with the $K$ buffered rollouts of prompt $q$:
\begin{equation}
\label{eq:ess}
    \mathrm{ESS}_q
    =
    \frac{\left(\sum_{k=1}^{K} w_k\right)^{2}}{\sum_{k=1}^{K} w_k^{2}}
    \;\in\;[1,\,K].
\end{equation}
$\mathrm{ESS}_q$ approaches $K$ when the buffered behavior policies are close to the current policy, and drops toward $1$ when the weights are dominated by a single stale rollout. In the fusion step of Algorithm~\ref{alg:ours}, the SNIS estimate is accepted only if $\mathrm{ESS}_q\geq\tau_{\mathrm{ESS}}$ (we use $\tau_{\mathrm{ESS}}=3$ throughout); otherwise the estimator falls back to the embedding-based teacher prediction from the cold-start phase. \autoref{tab:ess_threshold_sweep} reports the coverage (fraction of buffered prompts whose SNIS estimate is accepted) and MAE on the $200$ held-out prompts of \autoref{tab:difficulty_accuracy} as $\tau_{\mathrm{ESS}}$ varies over $\{1,2,3,4,5\}$; $\tau_{\mathrm{ESS}}=3$ minimizes MAE ($0.189$) at a coverage of $88.7\%$, striking the best trade-off between rejecting weight-collapsed estimates and leaving enough prompts under SNIS control to matter.

\begin{table}[h]
\centering
\small
\caption{Effect of the ESS acceptance threshold $\tau_{\mathrm{ESS}}$ on the SNIS difficulty estimator (Qwen2.5-Math-1.5B, same $200$ held-out prompts as \autoref{tab:difficulty_accuracy}, group size $G=8$ so $\mathrm{ESS}_q\in[1,8]$). Coverage is the fraction of prompts for which the SNIS estimate is accepted; rejected prompts fall back to the cold-start embedding predictor. Lower MAE is better.}
\label{tab:ess_threshold_sweep}
\begin{tabular}{cccccc}
\toprule
$\tau_{\mathrm{ESS}}$ & 1 & 2 & 3 & 4 & 5 \\
\midrule
Coverage (\%) & 99.1 & 95.0 & 88.7 & 76.4 & 46.2 \\
MAE$\downarrow$ & 0.2190 & 0.2010 & \textbf{0.1889} & 0.2185 & 0.2498 \\
\bottomrule
\end{tabular}
\end{table}

\subsection{Case study: stratified generation behavior.}
\label{sec:case_study}

  \autoref{fig:case_study} illustrates the stratified generation behavior with two representative examples from the MATH benchmark~\citep{hendrycks2021measuring}.

  \begin{figure}[h]
  \centering
  \small
  \fbox{%
  \begin{minipage}{0.96\linewidth}
  \textbf{Level 1 (Easy)} \hfill \textit{Long: 334 tokens \quad Short: 219 tokens}

  \textbf{Problem:} Find $k$, if ${(3^k)}^6=3^6$.

  \medskip
  \textbf{Long rollout (334 tokens, \cmark):}

  \begin{quote}
  To solve $(3^k)^6 = 3^6$, we use $(a^m)^n = a^{m\cdot n}$:
  \[
  (3^k)^6 = 3^{6k}
  \]
  So $3^{6k} = 3^6$, thus $6k = 6$, giving $k = 1$.
  Substituting back: $(3^1)^6 = 3^6 \implies 3^6 = 3^6$. The solution is correct.
  We can also verify with code: \texttt{(3**1)**6 == 3**6} returns \texttt{True}. \quad $\boxed{1}$
  \end{quote}

  \textbf{Short rollout (219 tokens, \cmark):}

  \begin{quote}
  Using the power rule $(a^m)^n = a^{mn}$: $(3^k)^6 = 3^{6k}$.
  So $3^{6k} = 3^6 \Rightarrow 6k = 6 \Rightarrow k = 1$.
  Substituting back: $(3^1)^6 = 3^6 \implies 3^6 = 3^6$. \quad $\boxed{1}$
  \end{quote}
  \end{minipage}}

  \vspace{6pt}

  \fbox{%
  \begin{minipage}{0.96\linewidth}
  \textbf{Level 5 (Hard)} \hfill \textit{Short wrong vs.\ long correct rollout}

  \textbf{Problem:} There are 360 people in a school. 15 take calculus, physics, and chemistry; 15 take none. 180 take calculus. Twice as many take chemistry as physics. 75 take both calculus and chemistry; 75 take both physics and chemistry; 30 take both physics and calculus. How many take physics?

  \medskip
  \textbf{Short rollout (427 tokens, \xmark):}

  \begin{quote}
  Using inclusion-exclusion: $180 + P + 2P - (15 + 75 + 75) + 30 = 345$,
  so $45 + 3P = 345$, giving $P = 100$. \quad $\boxed{100}$ \quad \textit{(wrong: misassigns $|C\cap P|=15$ and $|C\cap P\cap Ch|=30$)}
  \end{quote}

  \textbf{Long rollout (970 tokens, \cmark):}

  \begin{quote}
  Let $|C|=180$, $|P|$, $|Ch|=2|P|$, $|C\cap P\cap Ch|=15$, $|C\cap Ch|=75$, $|P\cap Ch|=75$, $|C\cap P|=30$, $|C\cup P\cup Ch|=345$.
  By inclusion-exclusion:
  \[
  345 = 180 + |P| + 2|P| - 30 - 75 - 75 + 15 \;\Rightarrow\; 330 = 3|P| \;\Rightarrow\; |P| = 110.
  \]
  \quad $\boxed{110}$
  \end{quote}
  \end{minipage}}

  \caption{Case study of stratified generation behavior. For the easy Level-1 problem, the model produces a concise 219-token solution. For the hard Level-5 problem, a short 427-token rollout misassigns intersection values and gives the wrong answer (100), while a longer 970-token rollout carefully applies inclusion-exclusion and arrives at the correct answer (110).}
  \label{fig:case_study}
  \end{figure}

  The Level-1 problem is solved correctly by all 64 rollouts; the shortest correct response uses only 219 tokens, demonstrating that the EasyLength penalty successfully encourages concise reasoning on easy prompts. In contrast, the Level-5 problem has an accuracy of only 23/64: the shorter 427-token rollout misassigns the intersection sizes and arrives at the wrong answer (100 instead of 110), while the longer 970-token rollout carefully applies the inclusion-exclusion principle and reaches the correct answer. This contrast illustrates how the think-longer re-rollout for hard prompts---generating extra rollouts under a hint-augmented prompt and replacing incorrect responses with correct ones---provides richer gradient signal that pure filtration cannot supply.

\subsection{Entropy--Difficulty Relationship Across Models}
\label{sec:entropy_difficulty}

\begin{figure}[h]
\centering
\includegraphics[width=1.0\textwidth]{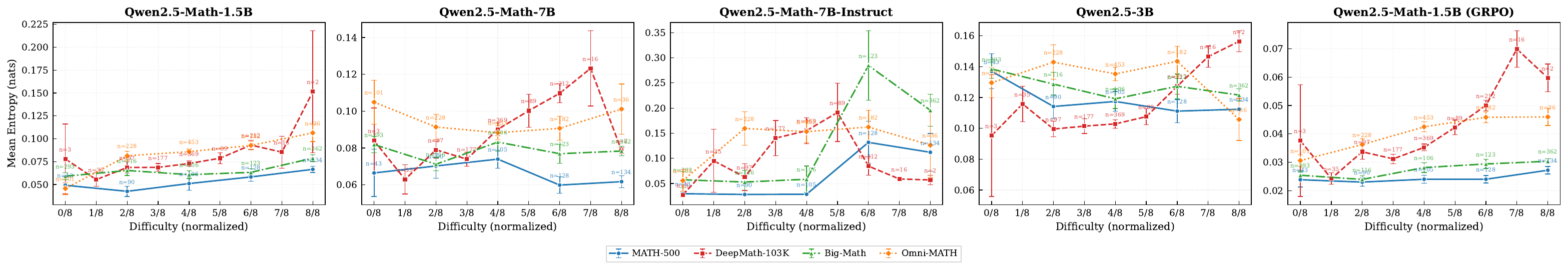}
\caption{Reasoning entropy versus normalized difficulty across multiple models and datasets, adapted from DiffAdapt~\cite{liu2026diffadapt}. The overall tendency is compatible with lower entropy in the middle difficulty range, although the magnitude and consistency of the effect vary across settings.}
\label{fig:entropy_bymodel_all}
\end{figure}

\autoref{fig:entropy_bymodel_all} visualizes the relationship between sequence-level reasoning entropy and normalized prompt difficulty across four models (Qwen2.5-Math-1.5B, Qwen2.5-3B, Qwen2.5-Math-7B, and Qwen2.5-7B) on multiple datasets. Across most settings, entropy tends to be lower for prompts of moderate difficulty, consistent with the intuition that mid-difficulty prompts elicit more focused reasoning patterns. However, the magnitude and consistency of this trend vary substantially across model--dataset combinations, suggesting that entropy alone is an unreliable proxy for adaptive difficulty---motivating the importance-sampling-based estimator proposed in \autoref{sec:is_estimation}.

\begin{table}[t]
\centering
\small
\caption{Effect of the easy penalty coefficient $\lambda_{\mathrm{easy}}$ and hard bonus coefficient $\lambda_{\mathrm{hard}}$ in the shaped reward on Qwen2.5-Math-1.5B. The top row is \ours without reward shaping. ``Tok'' denotes the average number of response tokens per problem at inference time.}
\label{tab:el_hl_sweep}
\resizebox{\textwidth}{!}{%
\begin{tabular}{cc cc cc cc cc cc}
\toprule
\multirow{2}{*}{$\lambda_{\mathrm{easy}}$} & \multirow{2}{*}{$\lambda_{\mathrm{hard}}$} &
\multicolumn{2}{c}{AIME24/25} &
\multicolumn{2}{c}{GSM8K} & \multicolumn{2}{c}{AIME-AMC} &
\multicolumn{2}{c}{Minerva} & \multicolumn{2}{c}{Olympiad} \\
\cmidrule(lr){3-4}\cmidrule(lr){5-6}\cmidrule(lr){7-8}\cmidrule(lr){9-10}\cmidrule(lr){11-12}
 & & Acc & Tok & Acc & Tok & Acc & Tok & Acc & Tok & Acc & Tok \\
\midrule
-- & --                                  & \textbf{8.33} & 1266        & 80.14          & 370 & 46.99          & 868  & 14.34          & 746  & 24.04          & 920  \\
$10^{-4}$ & $10^{-4}$                  & 3.33          & 1466  & 82.49          & 351 & 46.99          & 1056 & \textbf{17.28} & 882  & \textbf{25.07} & 1062 \\
$10^{-4}$ & --                         & 6.67          & \textbf{1000}          & 86.43          & \textbf{233} & 46.99          & 796  & 12.87          & 720  & 21.96          & 764  \\
$2{\times}10^{-4}$ & $10^{-4}$         & 3.33          & 1104         & \textbf{89.01}          & 248 & 44.58          & \textbf{741}  & 12.87          & \textbf{712}  & 21.07          & \textbf{667}  \\
$10^{-4}$ & $2{\times}10^{-4}$         & 6.67          & 2011        & 80.52          & 588 & \textbf{48.19} & 1842 & 16.91          & 1515 & 24.04          & 1491 \\
% $10^{-3}$ & --                         & 0.00          & 697  & 65.00          & \textbf{89.99} & 171 & 42.17          & 458  & 11.03          & 333  & 14.99          & 539  \\
\bottomrule
\end{tabular}}
\end{table}

\begin{table}[t]
\centering
\small
\caption{Ablation on asymmetric clipping range for easy prompts on Qwen2.5-Math-1.5B. The clipping bounds are $1-\epsilon_q^{-}$ and $1+\epsilon_q^{+}$; the top row is \ours with symmetric clip ranges. Relaxing only the upper bound (0.8--1.6) yields the best overall accuracy.}
\label{tab:clip_sweep}
\begin{tabular}{lcccccc}
\toprule
Clip Range & AIME24/25 & MATH-500 & GSM8K & AIME-AMC & Minerva & Olympiad \\
\midrule
0.8--1.2 & 8.33 & 69.80 & 80.14 & 46.99 & 14.34 & 24.04 \\
0.6--1.2 & 6.67 & 69.32 & 78.82 & 45.54 & 11.70 & 23.41 \\
0.6--1.4 & 8.33 & 70.40 & 81.79 & 48.80 & 17.65 & 24.04 \\
0.8--1.4 & 8.33 & 70.60 & 82.34 & 49.40 & 18.75 & 24.04 \\
0.8--1.6  & 10.00 & 70.84 & 81.20 & 50.12 & 19.99 & 24.04 \\
\bottomrule
\end{tabular}
\end{table}

% \subsection{Learning Rate and Scheduler Sensitivity}
% \label{sec:loss_ablation}

% \begin{figure}[h]
% \centering
% \includegraphics[width=0.7\textwidth]{figure/loss_ablation_math500.pdf}
% \caption{MATH-500 accuracy over training steps under different learning rates ($2{\times}10^{-4}$, $10^{-3}$, $2{\times}10^{-3}$), with and without a cosine schedule (Qwen2.5-Math-1.5B, trained on the same four-dataset union as the main results). \ours's gains are consistent across all four configurations, indicating that the improvements do not hinge on a specific optimizer schedule.}
% \label{fig:loss_ablation_MATH-500}
% \end{figure}

% \subsection{Additional Ablation Studies}
% \label{sec:ablation_studies}

\subsection{Ablation results on Qwen2.5-Math-1.5B.}
\label{sec:ablation_1_5b}

\autoref{tab:ablation_1_5b} isolates the contribution of each component on Qwen2.5-Math-1.5B by adding each design independently to the \textit{GRPO} baseline, and then evaluates their combined effect in \ours. The results show that each component improves over \textit{GRPO} on overall accuracy, with gains ranging from $+1.8$ for \textit{SNIS} to $+3.5$ for the hard bonus. Dynamic data selection, the easy penalty, and the hard bonus provide consistent gains across most benchmarks, indicating that accurate difficulty estimation, difficulty-aware sampling, and difficulty-conditioned reward shaping each contribute to better reasoning performance. Combining all components yields the strongest results on every benchmark, improving the overall score from $45.2$ to $50.8$ ($+5.6$), which confirms that the components are complementary rather than redundant.

\begin{table}[h]
\centering
\caption{Per-component ablation on Qwen2.5-Math-1.5B, mirroring \autoref{tab:main_results}. Each row adds one component \emph{independently} on top of the \textit{GRPO} baseline; the bottom row \emph{\ours (all components)} combines all components. \textcolor{green!50!black}{\small$\blacktriangle$} / \textcolor{red!70!black}{\small$\blacktriangledown$} denotes improvement / degradation over \textit{GRPO}.}
\label{tab:ablation_1_5b}
\resizebox{\textwidth}{!}{%
\begin{tabular}{lcccccc}
\toprule
Method & MATH-500 & GSM8K & AIME-AMC & Minerva & Olympiad & Overall \\
\midrule
\textit{GRPO} & 66.8 & 77.6 & 45.8 & 13.2 & 22.6 & 45.2 \\
\quad + \textit{SNIS} & 69.8\,{\scriptsize\up{3.0}} & 80.1\,{\scriptsize\up{2.5}} & 46.7\,{\scriptsize\up{0.9}} & 14.3\,{\scriptsize\up{1.1}} & 24.0\,{\scriptsize\up{1.4}} & 47.0\,{\scriptsize\up{1.8}} \\
\quad + \iconeasy\iconmedium\iconhard dynamic data selection & 71.6\,{\scriptsize\up{4.8}} & 83.4\,{\scriptsize\up{5.8}} & 47.0\,{\scriptsize\up{1.2}} & 15.1\,{\scriptsize\up{1.9}} & 24.2\,{\scriptsize\up{1.6}} & 48.3\,{\scriptsize\up{3.1}} \\
% \quad + \iconhard Think-longer re-rollout & 71.8\,{\scriptsize\up{5.0}} & 81.1\,{\scriptsize\up{3.5}} & 48.2\,{\scriptsize\up{2.4}} & 15.1\,{\scriptsize\up{1.9}} & 25.4\,{\scriptsize\up{2.8}} & 48.3\,{\scriptsize\up{3.1}} \\
\quad + \iconhard easy penalty & 69.4\,{\scriptsize\up{2.6}} & 83.5\,{\scriptsize\up{5.9}} & 48.2\,{\scriptsize\up{2.4}} & 16.2\,{\scriptsize\up{3.0}} & 24.0\,{\scriptsize\up{1.4}} & 48.5\,{\scriptsize\up{3.3}} \\
\quad + \iconhard hard bonus & 71.6\,{\scriptsize\up{4.8}} & 82.5\,{\scriptsize\up{4.9}} & 47.0\,{\scriptsize\up{1.2}} & 17.3\,{\scriptsize\up{4.1}} & 25.1\,{\scriptsize\up{2.5}} & 48.7\,{\scriptsize\up{3.5}} \\
% \quad + \iconeasy Relaxed upper clipping & 70.6\,{\scriptsize\up{3.8}} & 82.3\,{\scriptsize\up{4.7}} & 49.4\,{\scriptsize\up{3.6}} & 18.8\,{\scriptsize\up{5.6}} & 24.0\,{\scriptsize\up{1.4}} & 49.0\,{\scriptsize\up{3.8}} \\
\cmidrule(lr){1-7}
\ours (all components) & \textbf{72.6}\,{\scriptsize\up{5.8}} & \textbf{84.5}\,{\scriptsize\up{6.9}} & \textbf{50.6}\,{\scriptsize\up{4.8}} & \textbf{19.9}\,{\scriptsize\up{6.7}} & \textbf{26.4}\,{\scriptsize\up{3.8}} & \textbf{50.8}\,{\scriptsize\up{5.6}} \\
\bottomrule
\end{tabular}%
}
\end{table}

\subsection{Additional results on SmolLM3-3B.}
\label{sec:ablation_smollm}

\autoref{fig:training_curves_smollm} compares the training dynamics of \ours and baseline methods on SmolLM3-3B-Base, using both training steps and wall-clock time as the optimization budget. The results show that \ours consistently achieves higher accuracy than all baselines on both MATH-500 and GSM8K throughout training. The shaded regions indicate a clear and persistent performance gap over the strongest competing method, showing that the gain is not limited to the final checkpoint. Moreover, the wall-clock comparison on MATH-500 confirms that this advantage remains when accounting for actual training time, indicating that \ours improves both training efficiency and final effectiveness on SmolLM3-3B-Base.

\begin{figure}[ht]
\centering
\includegraphics[width=1.0\textwidth]{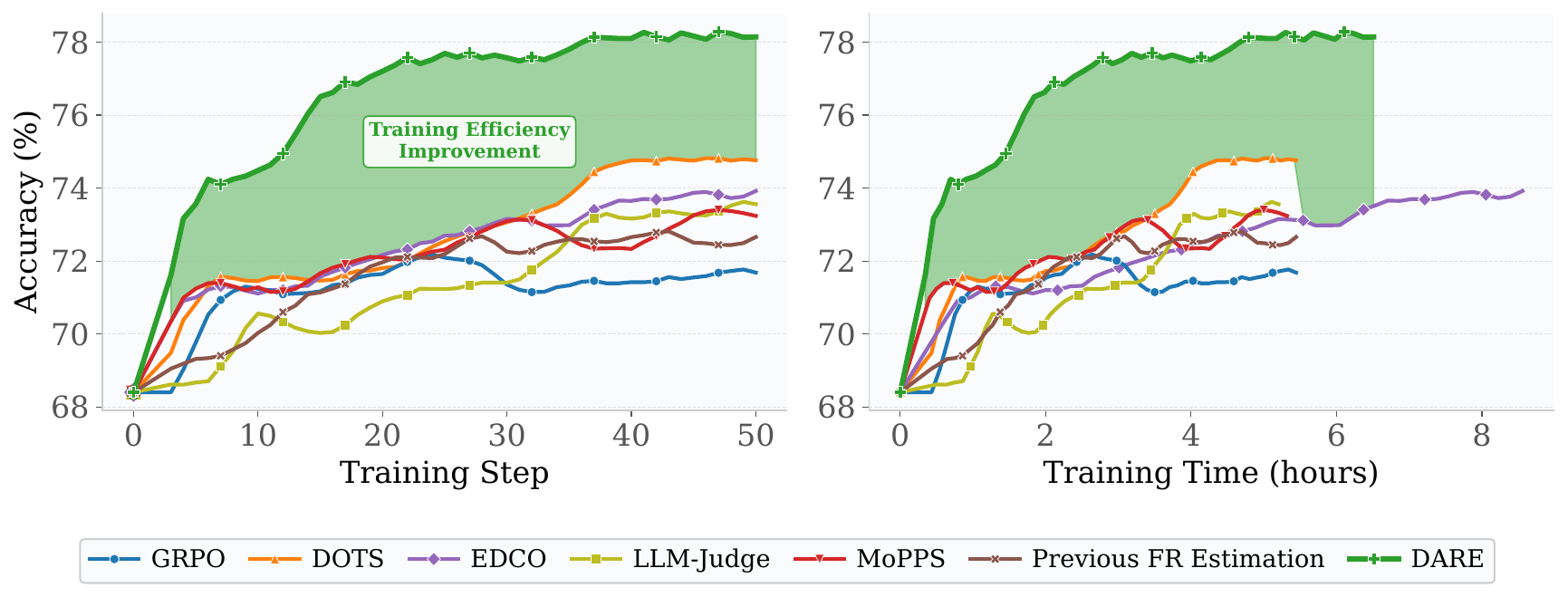}
\caption{Training curves on MATH-500 (left), GSM8K (center), and MATH-500 vs.\ training time (right) for SmolLM3-3B-Base. \ours consistently outperforms all baselines throughout training, with the shaded region highlighting the accuracy advantage over the best competing method. The training-time plot confirms that the advantage is preserved when the $x$-axis is wall-clock time rather than training steps.}
\label{fig:training_curves_smollm}
\end{figure}

\section{Generalization to Code Generation}
\label{sec:coder_transfer}

To verify that the SNIS-over-baselines ordering observed on mathematical reasoning also holds on a domain with execution-based rewards, we train Qwen2.5-Coder-1.5B~\citep{hui2024qwen2coder} under the same set of difficulty-estimation baselines used in \autoref{sec:experiments}. Training data consists of $10{,}240$ problems drawn in a 50/50 split from TACO~\citep{li2023taco} (BAAI) and APPS~\citep{hendrycks2021apps} (codeparrot), with per-problem stdin/stdout test cases used as the verifiable reward. Evaluation covers the standard code-generation setup on HumanEval~\citep{chen2021humaneval} and MBPP~\citep{austin2021mbpp} together with their EvalPlus hardened variants HumanEval+ and MBPP+~\citep{liu2023evalplus}, and the contamination-controlled LiveCodeBench~\citep{jain2025livecodebench_iclr}.

\autoref{tab:coder_final_acc} reports final accuracy on the four benchmarks. SNIS difficulty estimation is the top row on HumanEval (46.9) and MBPP (71.4), and is within $0.9$ points of the best competitor on HumanEval+ and MBPP+. \autoref{fig:training_curves_coder} shows the corresponding training trajectories, where SNIS maintains a consistent lead on HumanEval and MBPP throughout training and tracks the top tier on the plus variants. Together, these results indicate that the estimator-level advantage transfers from mathematical reasoning to code generation.

\begin{table}[h]
\centering
\small
\caption{Final accuracy on code-generation benchmarks on Qwen2.5-Coder-1.5B, trained on TACO\,+\,APPS (50/50, $10{,}240$ prompts). Results are reported as pass@1 (\%). The best per-column value is shown in bold.}
\label{tab:coder_final_acc}
\begin{tabular}{lcccc}
\toprule
Method & HumanEval  & MBPP  \\
\midrule
\textit{GRPO} & 44.96 & 69.32 \\
\textit{DOTS} & 43.70  & 68.37 \\
\textit{EDCO} & 43.23 & 70.17 \\
\textit{LLM-Judge} & 45.80 & 70.48  \\
\textit{MoPPS} & 46.17 & 69.29\\
\textit{Previous FR Estimation} & 45.58  & 69.72 \\
\midrule
% \textit{Current FR Estimation} & 45.60 & 38.70 & 71.37 & 59.34 \\
\ours (Ours) & \textbf{46.86} & \textbf{71.44}  \\
\bottomrule
\end{tabular}
\end{table}

\begin{figure}[h]
\centering
\includegraphics[width=1.0\textwidth]{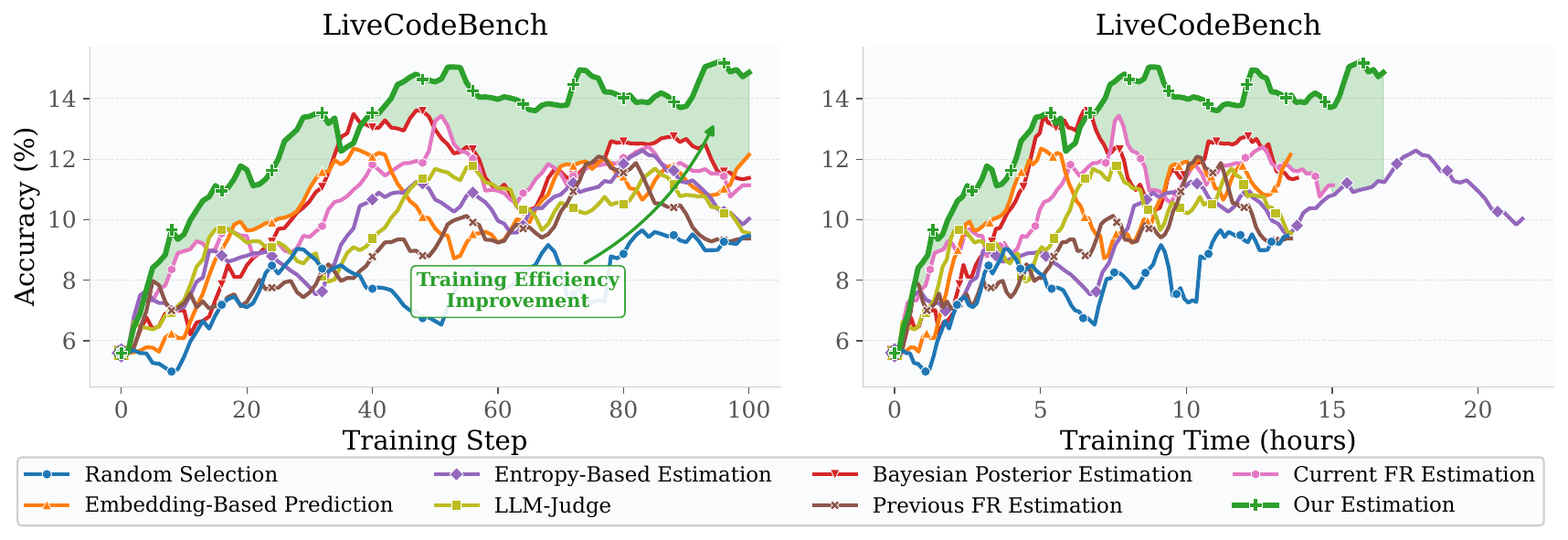}
\caption{Training curves on code-generation benchmarks for Qwen2.5-Coder-1.5B \cite{hui2024qwen2coder}. Our difficulty estimation tracks or exceeds all baselines throughout training, with the largest relative advantage on HumanEval and MBPP.}
\label{fig:training_curves_coder}
\end{figure}

\section{Training Efficiency Analysis}
\label{sec:training_efficiency}

\paragraph{Metrics.}
To go beyond visual inspection of the training curves in \autoref{fig:training_curves_7b} and \autoref{fig:training_curves_smollm}, we summarize each training run with two budget-level metrics defined directly on the per-epoch accuracy trajectory $\mathrm{acc}_m(s)$ of method $m$ at training step $s$. Let $H$ denote the common horizon (the last step reached by every method on a given benchmark). We define
\begin{equation}
\label{eq:steps_to_target}
S_m(\tau)=\min\bigl\{s\in[0,H]\colon \mathrm{acc}_m(s)\ge\tau\bigr\},
\qquad
\mathrm{sp}_m(\tau)=\frac{S_{\mathrm{base}}(\tau)}{S_m(\tau)},
\end{equation}
the step at which $m$ first reaches target accuracy $\tau$ and the corresponding \emph{step speedup} over the \textit{Random Selection} baseline. We further define the normalized area under the accuracy curve,
\begin{equation}
\label{eq:norm_auc}
\mathrm{AUC}_m \;=\; \frac{1}{H}\int_{0}^{H}\mathrm{acc}_m(s)\,ds
\;\approx\; \frac{1}{H}\sum_{s=0}^{H-1}\tfrac{1}{2}\!\bigl(\mathrm{acc}_m(s)+\mathrm{acc}_m(s{+}1)\bigr),
\end{equation}
which equals the average accuracy attained throughout training and is insensitive to evaluation noise at any single step. Unlike final accuracy, $\mathrm{AUC}_m$ jointly penalizes slow starts and early plateaus, so a higher value means both \emph{how fast} and \emph{how high} the curve climbs are favorable.

\paragraph{Results on Qwen2.5-Math-1.5B and Qwen2.5-Math-7B.}
\autoref{tab:training_efficiency} reports $S_m(\tau)$, $\mathrm{sp}_m(\tau)$, and $\mathrm{AUC}_m$ on Qwen2.5-Math-1.5B and Qwen2.5-Math-7B for MATH-500 and GSM8K. SNIS reaches the highest MATH-500 target on Qwen2.5-Math-1.5B ($67\%$) in $42$ steps, compared to $118$ steps for \textit{Random Selection} ($\mathrm{sp}{=}2.81\times$); on Qwen2.5-Math-7B MATH-500 the highest target ($75\%$) is reached in $14$ steps vs.\ $122$ ($\mathrm{sp}{=}8.71\times$). In terms of $\mathrm{AUC}$, SNIS is the top row on Qwen2.5-Math-1.5B MATH-500 and ties the top tier ($\mathrm{AUC}$ within $0.2$ of the best) on Qwen2.5-Math-1.5B GSM8K and Qwen2.5-Math-7B MATH-500. On Qwen2.5-Math-7B GSM8K, \textit{MoPPS} attains the highest $\mathrm{AUC}$ ($88.8$ vs.\ SNIS $86.5$), consistent with the pattern in \autoref{fig:training_curves_7b} where GSM8K saturates early for all filtration-based methods.

\begin{table}[h]
\centering
\small
\caption{Training efficiency of SNIS difficulty estimation versus filtration-only baselines on Qwen2.5-Math-1.5B and Qwen2.5-Math-7B, computed from the per-epoch accuracy trajectories used in \autoref{fig:training_curves_7b}. $\mathrm{AUC}$ is the normalized area under the accuracy curve (\autoref{eq:norm_auc}, as a percentage). $S_\tau$ is the smallest training step at which a method first reaches accuracy $\tau$ (\autoref{eq:steps_to_target}); ``--'' means the target is not reached within the common horizon. $\mathrm{sp}_\tau$ is the step speedup $S_{\mathrm{Random}}(\tau)/S_m(\tau)$; on Qwen2.5-Math-7B GSM8K, \textit{Random Selection} does not reach $\tau{=}0.88$, so we report $S_{88}$ only.}
\label{tab:training_efficiency}
\resizebox{\textwidth}{!}{%
\begin{tabular}{l cc cc cc cc}
\toprule
& \multicolumn{4}{c}{\textit{Qwen2.5-Math-1.5B}} & \multicolumn{4}{c}{\textit{Qwen2.5-Math-7B}} \\
\cmidrule(lr){2-5}\cmidrule(lr){6-9}
& \multicolumn{2}{c}{MATH-500 (H=194)} & \multicolumn{2}{c}{GSM8K (H=194)} & \multicolumn{2}{c}{MATH-500 (H=198)} & \multicolumn{2}{c}{GSM8K (H=198)} \\
\cmidrule(lr){2-3}\cmidrule(lr){4-5}\cmidrule(lr){6-7}\cmidrule(lr){8-9}
Method & AUC & $S_{67}/\mathrm{sp}_{67}$ & AUC & $S_{78}/\mathrm{sp}_{78}$ & AUC & $S_{75}/\mathrm{sp}_{75}$ & AUC & $S_{88}$ \\
\midrule
\textit{Random Selection} & 64.4 & 118 / 1.00$\times$ & 72.4 & 184 / 1.00$\times$ & 72.2 & 122 / 1.00$\times$ & 84.3 & -- \\
\textit{DOTS} & 65.7 & 64 / 1.84$\times$ & 75.5 & 74 / 2.49$\times$ & \textbf{73.3} & 52 / 2.35$\times$ & 86.3 & 56 \\
\textit{EDCO} & 65.5 & 86 / 1.37$\times$ & 75.2 & 104 / 1.77$\times$ & 73.2 & 52 / 2.35$\times$ & 86.1 & \textbf{12} \\
\textit{LLM-Judge} & 65.1 & 102 / 1.16$\times$ & 74.1 & 124 / 1.48$\times$ & 72.8 & 52 / 2.35$\times$ & 85.4 & 162 \\
\textit{MoPPS} & 65.1 & 92 / 1.28$\times$ & \textbf{77.3} & 54 / 3.41$\times$ & 72.7 & 114 / 1.07$\times$ & \textbf{88.8} & 30 \\
\textit{Previous FR Estimation} & 64.1 & 114 / 1.04$\times$ & 73.1 & 82 / 2.24$\times$ & 73.0 & 118 / 1.03$\times$ & 86.1 & 58 \\
\ours (Ours) & \textbf{66.3} & \textbf{42 / 2.81$\times$} & 77.1 & \textbf{46 / 4.00$\times$} & 73.2 & \textbf{14 / 8.71$\times$} & 86.5 & 36 \\
\bottomrule
\end{tabular}%
}
\end{table}

\paragraph{Results on SmolLM3-3B.}
\autoref{tab:training_efficiency_smollm} reports the same metrics on SmolLM3-3B-Base, computed from the curves in \autoref{fig:training_curves_smollm}. The common horizon is $H{=}98$ steps; targets are chosen close to the plateau of \textit{Random Selection} so that the step speedups remain informative. SNIS is the top row on both MATH-500 and GSM8K in $\mathrm{AUC}$ and reaches both targets in the fewest steps.

\begin{table}[h]
\centering
\small
\caption{Training efficiency on SmolLM3-3B-Base (same metrics as \autoref{tab:training_efficiency}). \textit{Random Selection} does not reach $\tau{=}0.73$ on MATH-500 or $\tau{=}0.85$ on GSM8K within the common horizon, so step speedups are undefined and we report the absolute step $S_\tau$ only.}
\label{tab:training_efficiency_smollm}
\begin{tabular}{l cc cc}
\toprule
& \multicolumn{2}{c}{MATH-500 (H=98)} & \multicolumn{2}{c}{GSM8K (H=98)} \\
\cmidrule(lr){2-3}\cmidrule(lr){4-5}
Method & AUC & $S_{73}$ & AUC & $S_{85}$ \\
\midrule
\textit{GRPO} & 68.6 & -- & 79.7 & -- \\
\textit{DOTS} & 70.4 & 44 & 80.7 & 22 \\
\textit{EDCO} & 69.9 & 44 & 82.2 & 14 \\
\textit{LLM-Judge} & 69.0 & 62 & 80.9 & 34 \\
\textit{MoPPS} & 69.6 & 50 & 82.1 & 16 \\
\textit{Previous FR Estimation} & 69.1 & 34 & 80.4 & 84 \\
% \textit{Current FR Estimation} & 69.6 & 48 & 80.3 & 70 \\
\ours (Ours) & \textbf{71.0} & \textbf{24} & \textbf{82.8} & \textbf{14} \\
\bottomrule
\end{tabular}
\end{table}

\section{Accuracy Under Constrained Rollout Budgets}
\label{sec:budget_accuracy_appendix}

To further validate the practical impact of improved difficulty estimation, we evaluate downstream accuracy under constrained rollout budgets on two challenging benchmarks: AIMO-AMC (83 problems) and MATH-500 Level~5 (134 problems). As shown in \autoref{fig:budget_accuracy}, our SNIS difficulty estimation method consistently achieves higher accuracy than \textit{Embedding-Based Prediction}, \textit{LLM-Judge}, and \textit{Uniform GRPO} baselines across all budget levels ($K{=}1,4,8,16,32,64$), with the advantage most pronounced at small budgets where accurate difficulty estimation matters most. \autoref{fig:budget_accuracy} also shows that the set of problems solved by SNIS difficulty estimation overlaps strongly with baseline coverage while still including additional uniquely solved examples.

\begin{figure}[h]
\centering
\begin{minipage}[t]{0.48\textwidth}
    \centering
    \includegraphics[width=\textwidth]{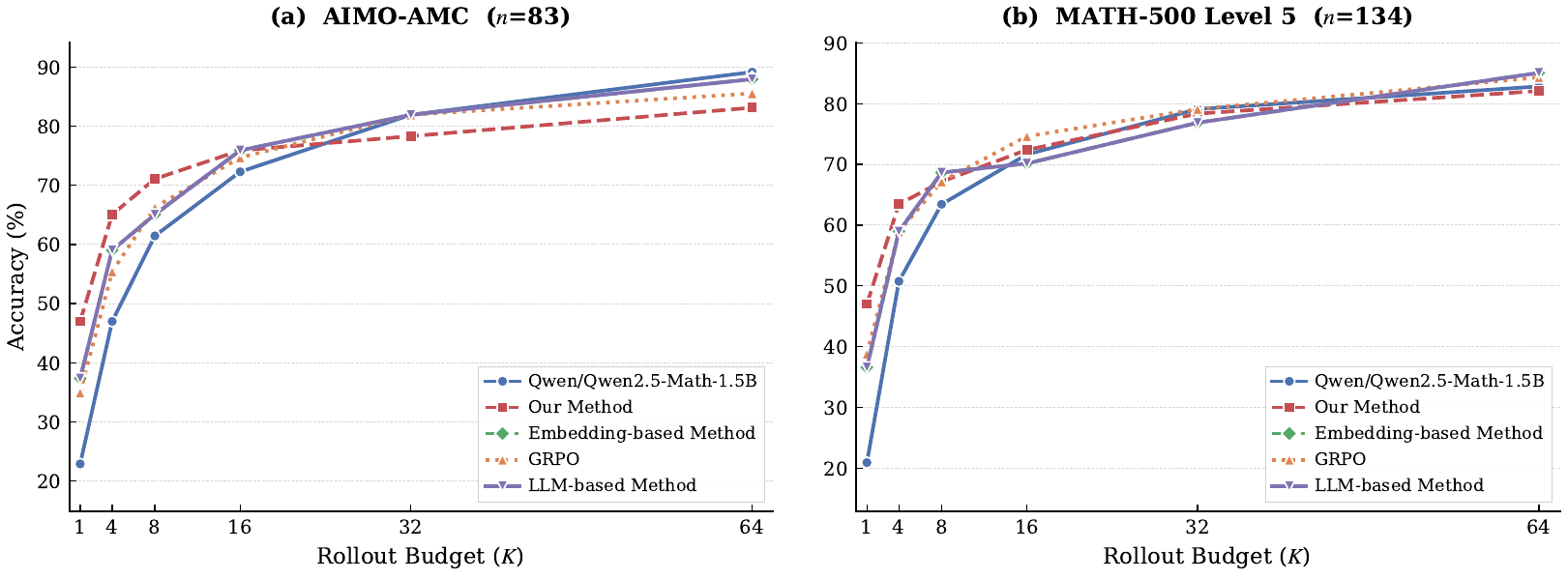}
\end{minipage}\hfill
\begin{minipage}[t]{0.48\textwidth}
    \centering
    \includegraphics[width=\textwidth]{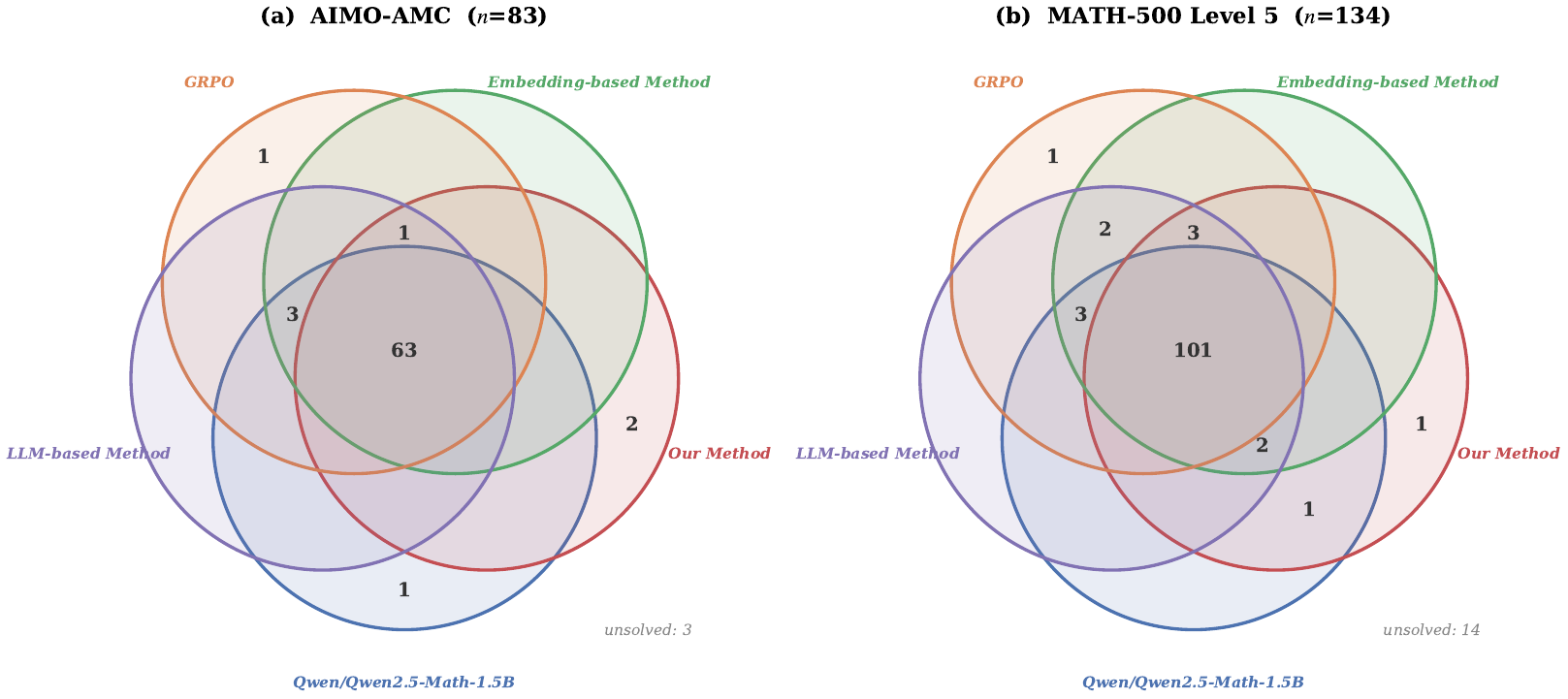}
\end{minipage}
\caption{\textbf{Left}: Accuracy under constrained rollout budgets on AIMO-AMC and MATH-500 Level~5. Our difficulty estimation method achieves consistently higher accuracy, with the largest gains at small budgets. \textbf{Right}: Venn diagram of problems solved at $K{=}64$. Our difficulty estimation covers most problems solved by any baseline while uniquely solving additional examples.}
\label{fig:budget_accuracy}
\end{figure}

\section{Impact Statement}
\label{sec:impact}

This work proposes a difficulty-adaptive RL fine-tuning framework that improves the sample efficiency of training LLM-based mathematical reasoners. On the positive side, reducing the number of rollouts required to reach a given accuracy level directly lowers the compute and energy cost of RL fine-tuning, which benefits researchers with smaller hardware budgets and reduces the environmental footprint of training strong reasoning models; it also makes it more tractable to iterate on reward design and curriculum strategies. On the negative side, our method makes capable mathematical-reasoning LLMs cheaper to obtain, which could amplify existing concerns around misuse of reasoning models (for example, automated assistance with academic assessments). These risks are properties of capable reasoning LLMs in general rather than risks introduced by the proposed training algorithm, and we do not release any new pretrained model weights as part of this work.

% \section{Limitations}

% Although \ours improves training efficiency, final accuracy, and inference-token efficiency across multiple models and benchmarks, several limitations remain. First, \ours relies on verifiable binary rewards to estimate prompt difficulty and shape training signals, which may limit its direct use in tasks where correctness is subjective or hard to verify. Third, the replay buffer and SNIS-based difficulty estimation introduce extra implementation complexity and require storing historical rollouts and behavior-policy log probabilities. While this overhead is modest in our experiments, its cost may become more significant at larger scales. Finally, the difficulty thresholds, rollout budgets, and reward-shaping coefficients are fixed in our current implementation. Automatically adapting these hyperparameters during training is an important direction for future work.